\documentclass[11pt, draftclsnofoot, onecolumn]{IEEEtran}
\usepackage{multirow}
\usepackage{amsfonts}%
\usepackage{setspace}
\usepackage{amssymb}%
\usepackage{booktabs,blindtext}
\usepackage[dvips]{graphicx}%
\usepackage{amsmath}%
\usepackage[usenames]{color}
\usepackage{algorithmic}
\usepackage{comment}
\newtheorem{theorem}{Theorem}

\newtheorem{lemma}{Lemma}

\newtheorem{assumption}{Assumption}

\ifodd 0
\newcommand{\revq}[1]{{\color{red}#1}}
\else
\newcommand{\revq}[1]{#1}
\fi

\ifodd 0
\newcommand{\rev}[1]{{\color{blue}#1}} 

\newcommand{\com}[1]{\textbf{\color{red}(COMMENT: #1)}} 
\newcommand{\clar}[1]{\textbf{\color{green}(NEED more work: #1)}}

\else
\newcommand{\rev}[1]{#1}

\newcommand{\com}[1]{}
\newcommand{\clar}[1]{}
\fi

\begin{document}
\title{Online Learning in Decentralized Multiuser Resource Sharing Problems}
\author{\IEEEauthorblockN{Cem Tekin,~\IEEEmembership{Student Member,~IEEE}, Mingyan Liu,~\IEEEmembership{Senior Member,~IEEE}\\
\thanks{A preliminary version of this work appeared in MILCOM 12.}
\thanks{C. Tekin and M. Liu are with the Electrical Engineering and Computer Science Department, University of Michigan, Ann Arbor, MI 48105, USA,
        {\{cmtkn,mingyan\}@eecs.umich.edu}}
}}

\maketitle

\begin{abstract}
In this paper, we consider the general scenario of resource sharing in a decentralized
system when the resource rewards/qualities are time-varying and unknown to the users, and using the same
resource by multiple users leads to reduced quality due to
resource sharing. Firstly, we consider a user-independent reward model with no communication between the users, where a user gets feedback about the congestion level in the resource it uses. Secondly, we consider user-specific rewards and allow costly communication between the users. The users have a cooperative goal of achieving the highest system utility. There are multiple obstacles in achieving this goal such as the decentralized nature of the system, unknown resource qualities, communication, computation and switching costs. We propose distributed learning algorithms with logarithmic regret with respect to the optimal allocation. Our logarithmic regret result holds under both i.i.d. and Markovian reward models, as well as under communication, computation and switching costs.
\end{abstract}

\begin{IEEEkeywords}
Online learning, multi-armed bandits, distributed algorithms, online resource sharing, regret, exploration-exploitation tradeoff
\end{IEEEkeywords}


\section{Introduction} \label{sec:intro}

In this paper, we consider the general multiuser online learning problem in a resource sharing setting, where the reward of a user from a resource depends on how many others are using that resource. We use the mathematical framework of multi-armed bandit problems, where a resource corresponds to an arm, which generates random rewards depending on the number of users using that resource, at discrete time steps. The goal of the users is to achieve a system objective, such as maximization of the expected total utility over time.

One of the major challenges in a decentralized system is the asymmetric information between the users. In cases where communication between the users is not allowed, each of them should act based on its own history of observations and actions. In such a case, without feedback about the choices of other users, it will be impossible for all users to coordinate on an arbitrary system objective. Therefore, we introduce minimal feedback about the actions of the other users. Basically, at each time step, this feedback gives to each user the number of users using the same resource with it. When such feedback is available, given that the rewards from resources are non-user specific, we show that system utility, i.e., the sum of rewards of all users, can be made logarithmically close in time to the utility from the best static allocation of resources. The non-user specific resources is a necessary element of this problem, since communication is not allowed between users, they can only asses the quality of a resource for another user, based on their own perceived qualities.

Another model we consider allows the users to communicate with each other and share their past observations but with a cost of communication. We consider general user-specific random resource rewards without the feedback setting discussed above. We propose distributed online learning algorithms that achieve logarithmic regret, when a lower bound on the performance gap, i.e., the difference between the total reward of the best and second-best allocation, of the total utility function is known. If the performance gap is not known, then we propose a way to achieve near-logarithmic regret.

In addition to the aforementioned models, our algorithms achieve the same order of regret even when we introduce computation and switching costs, where the computation cost models the time and resources it takes for the users to compute the estimated optimal allocation, while switching cost models the cost of changing the resource.

One of our motivating application is opportunistic spectrum access (OSA) in a cognitive radio network. Each user in this setting corresponds to a transmitter-receiver pair, and each resource corresponds to a channel over which data transmission takes place. We assume a slotted time model where channel sensing is done at the beginning of a time slot, data is transmitted during a slot, and transmission feedback is received at the end of a slot. Quality of a channel dynamically changes according to some unknown stochastic process at the end of each slot. Each user selects a channel at each time step, and receives a reward depending on the channel quality during that time slot. Since the channel quality is acquired through channel sensing, the user can only observe the quality of the channel it sensed in the current time slot. The channel selection problem of a user can be cast as a decision making problem under uncertainty about the distribution of the channel rewards, and uncertainty due to partially observed channels.

The regret (sometimes called weak regret) of an algorithm with respect to the best static allocation rule at time $T$ is the difference between total expected reward of the best static allocation rule by time $T$ and the total expected reward of the algorithm by time $T$. \rev{Note that a static allocation rule is an offline rule in which a fixed allocation is selected at each time step, since the feedback received from the observations are not taken into account.} The regret quantifies the rate of convergence to the best static allocation. As $T$ goes to infinity performance of any algorithm with sub-linear regret will converge in terms of its average reward to the optimal static allocation, while the convergence is faster for an algorithm with smaller regret. 

The main contribution of this paper is to show that order optimal resource allocation algorithms can be designed under both limited feedback and costly communication. As we illustrate in subsequent sections, if the users are given limited feedback about the number of simultaneous users on the same resource, then they can achieve logarithmic regret with respect to the optimal static allocation. \rev{Even when the users are fully decentralized, assuming a costly communication is possible between the users, they will still achieve logarithmic regret with respect to the optimal static allocation.}

The organization of the rest of the chapter is as follows. We discuss related work in Section \ref{sec:related}. In Section \ref{sec:probform} we give the problem formulation. We study the limited feedback model without communication between the users in Section \ref{sec:fb_alg}. Then, we consider the user-specific resource reward model with communication in Section \ref{sec:com_alg}. We provide an OSA application, and give numerical results in Section \ref{sec:num}. A discussion is given in Section \ref{sec:diss}, and final remarks are given in Section \ref{sec:conc}.

\section{Related work} \label{sec:related}

We can fit our online resource sharing problems in the multi-armed bandit framework. Therefore, the related literature mostly consists of multi-armed bandit problems.

The single-player multi-armed bandit problem is widely studied, and well understood. The seminal work \cite{lairobbins1985} considers the problem where there is a single player that plays one arm at each time step, and the reward process for an arm is an IID process whose probability density function (pdf) is unknown to the player, but lies in a known parameterized family of pdfs. Under some regularity conditions such as the denseness of the parameter space and continuity of the Kullback-Leibler divergence between two pdfs in the parameterized family of pdfs, the authors provide an asymptotic lower bound on the regret of any {\em uniformly good} policy. This lower bound is logarithmic in time which indicates that at least a logarithmic number of samples should be taken from each arm to decide on the best arm with a high probability. A policy is defined as {\em asymptotically efficient} if it achieves this lower bound, and the authors construct such a policy which is also an index policy.

This result is extended in \cite{anantharamvaraiya1987-1} to a single player with multiple plays. However, the complexity of deciding on which arm to play \rev{is shown to increase  linearly in time both in \cite{lairobbins1985} and \cite{anantharamvaraiya1987-1}; this makes the policy computationally infeasible}. This problem is addressed in \cite{agrawal1995} where sample-mean based index policies that achieve logarithmic order of regret are constructed. The complexity of a sample-mean based policy does not depend on time since the decision at each time step only depends on parameters of the preceding time step.  The proposed policies are order optimal, i.e., they achieve the logarithmic growth of regret in time, though they are not in general optimal with respect to the constant. In all the works cited above, the limiting assumption is that there is a known single parameter family of pdfs governing the reward processes in which the correct pdf of an arm reward process resides. Such an assumption virtually reduces the arm quality estimation problem to a parameter estimation problem.

This assumption is relaxed in \cite{auerbianchi2002}, where it is only assumed that the reward of an arm is drawn from an unknown distribution with a bounded support. An index policy, called the {\em upper confidence bound} (UCB1), is proposed; it is similar to the one in \cite{agrawal1995} \rev{and it achieves logarithmic order of regret uniform in time}. A modified version of UCB1 with a smaller constant of regret is proposed in \cite{auer2010ucb}. Another work \cite{garivier2011kl} proposed an index policy, KL-UCB, which is uniformly better than UCB1 in \cite{auerbianchi2002}. Moreover, it is shown to be asymptotically optimal for Bernoulli rewards. Authors in \cite{audibert2009exploration} consider the same problem as in \cite{auerbianchi2002}, but in addition take into account empirical variance of the arm rewards when deciding which arm to select. They provide a logarithmic upper bound on regret with better constants under the condition that suboptimal arms have low reward variance. Moreover, they derive probabilistic bounds on the variance of the regret by studying its tail distribution.

Another part of the literature is concerned with the case where the reward processes are  Markovian. We refer to this as the {\em the Markovian model}. This offers a richer framework for the analysis, and is better suited for many real-world applications including opportunistic spectrum access, the central motivation underlying this chapter. The Markovian model can be further divided into two cases.

The first case is the {\em rested} Markovian model, in which the state of an arm evolves according to a Markov rule when it is played or activated, but otherwise remains frozen. A usual assumption under this model is that the reward process for each arm is modeled as a finite-state, irreducible, aperiodic Markov chain. This problem is first addressed in \cite{anantharamvaraiya1987-2} under a parameterized transition model, where {\em asymptotically efficient} index policies with logarithmic regret with respect to the \rev{optimal policy with known transition matrices, i.e., the policy that always selects the best arm,} is proposed. Similar to the work \cite{auerbianchi2002} under the IID model, \cite{tekin2010online} relaxes the parametric assumption, and proves that the index policy in \cite{auerbianchi2002} can achieve logarithmic regret in the case of rested Markovian reward by using a large deviation bound on Markov chains from \cite{lezaud1998chernoff}.

The second case is the {\em restless} Markovian model, in which the state of an arm evolves according to a Markov rule when it is played, and otherwise evolves according to an arbitrary process. This problem is significantly harder than the {\em rested} case; even when the transition probabilities of the arms are known a priori, it is PSPACE hard to approximate \rev{the optimal policy} \cite{papadim1999}. \rev{In \cite{tekin2011online} a regenerative cycle based algorithm which reduces the problem to estimating the mean reward of the arms by exploiting the regenerative cycles of the Markov process is proposed. A logarithmic regret bound with respect to the best static policy is proved} when the reward from each arm follows a finite-state, irreducible aperiodic Markov chain, and when there is a special condition on the {\em multiplicative symmetrization} of the transition probability matrix. A parallel development, \cite{liuliu2010} uses the idea of geometrically growing exploration and exploitation block lengths to prove a logarithmic regret bound. Both of the above studies utilize sample-mean based index policies that require minimum number of computations. Stronger measures of performance are studied in \cite{tekin2011approx} and \cite{tekin2011adaptive}. Specifically, \cite{tekin2011approx} considers an approximately optimal, computationally efficient algorithm for a special case of the restless bandit problem which is called the {\em feedback bandit} problem studied in \cite{guha2010approximation}. In \cite{guha2010approximation}, a computationally efficient algorithm for the {\em feedback bandit} problem with known transition probabilities is developed. This development follows from taking the Lagrangian of Whittle's LP \cite{whittle1988restless}, which relaxes the assumption that one arm is played at each time step, to one arm is played on average. 
The idea behind the algorithm in \cite{tekin2011approx} is to combine learning and optimization by using a threshold variant of the optimization policy proposed in \cite{guha2010approximation} on the estimated transition probabilities in exploitation steps. In \cite{tekin2011adaptive}, an algorithm with logarithmic regret with respect to best {\em dynamic} policy is given, under a general Markovian model. This algorithm solves the estimated average reward optimality equation (AROE), and assigns an index to each arm which maximizes the $h$ function over a ball centered around true transition probabilities. The drawback of this algorithm is that it is computationally inefficient. 

It should be noted that the distinction between rested and restless arms does not arise when the award process is IID.  This is because since the rewards are independently drawn each time, whether an unselected arm remains still or continues to change does not affect the reward this arm produces the next time it is played whenever that may be. This is clearly not the case with Markovian rewards. In the rested case, since the state is frozen when an arm is not played, the state in which we next observe the arm is independent of how much time elapses before we play the arm again. In the restless case, the state of an arm continues to evolve, thus the state in which we next observe it  now becomes dependent on the amount of time that elapses between two plays of the same arm. This makes the problem significantly more difficult.

Different from the single-player bandit problems, in a multi-player bandit problem decentralization of information and potential lack of coordination among the players play a crucial role in designing efficient algorithms. It is desirable to design distributed algorithms that guide the players to optimal allocations. Below we review several results from the multi-player bandit literature.

Most of the relevant work in decentralized multi-player bandit problems assumes that the optimal configuration of players on arms is such that at any time step there is at most one player on an arm. We call such a configuration an {\em orthogonal configuration}. \cite{liu2010distributed} and \cite{anandkumar2011distributed} consider the problem under the IID model and derive logarithmic upper and lower bounds for the regret assuming that the optimal configuration is an orthogonal one. Specifically, the algorithm in \cite{liu2010distributed} uses a mechanism called the {\em time division fair sharing}, where a player shares the best arms with the others in a predetermined order.  By contrast, in \cite{anandkumar2011distributed} the algorithm uses randomization to settle to an orthogonal configuration, which does not require predetermined ordering, at the cost of fairness. In the long run, each player settles down to a different arm, but the initial probability of settling to the best arm is the same for all players. The restless multi-player Markovian model is considered in \cite{tekin2012online} and \cite{liuliu2011}, where logarithmic regret algorithms are proposed. The method in \cite{tekin2012online} is based on the regenerative cycles \rev{which is developed for single player in \cite{tekin2011online}}, while the method in \cite{liuliu2011} is based on deterministic sequencing of exploration and exploitation blocks \rev{which is developed for single player in \cite{liuliu2010}}.

Another line of work considers combinatorial bandits on a bipartite graph, in which the goal is to find the best bipartite matching of users to arms. These bandits can model resource allocation problems when the resource qualities are user-dependent, but resource sharing is not allowed. Specifically, centralized multi-user combinatorial bandits is studied in \cite{gai2012combinatorial} and \cite{gai2012restless}, while a decentralized multi-user combinatorial bandit is studied in \cite{jain2012-1}. Due to the special structure of this problem, the optimal matching can be computed efficiently, whereas computing the optimal allocation is NP-hard in a general combinatorial bandit.

Although the assumption of optimality of an orthogonal configuration is suitable for applications such as random access or communication with collision models, it lacks the generality for applications where optimal allocation may involve sharing of the same channel; \rev{these include dynamic spectrum allocation based on techniques like code division multiple access (CDMA) and power control.} \rev{Our results in this paper reduce to the results in \cite{liuliu2011} when the optimal allocation is orthogonal.}

\section{Problem formulation and preliminaries} \label{sec:probform}

We consider $M$ decentralized users indexed by the set ${\cal M} = \{1,2,\ldots,M\}$, and $K$ resources indexed by the set ${\cal K} = \{1,2,\ldots,K\}$, in a discrete time setting $t=1,2,\ldots$.
Quality/reward of a resource varies stochastically over discrete time steps, and depends on the number of users using the resource. \revq{Each resource $k$ has an internal state $s^k_t$ which varies over time in according to an i.i.d. or Markovian rule. The quality/reward of a resource depends on the internal state of the resource an the number of users using the resource.}
In the i.i.d. model, the state of each resource follows an i.i.d. process with $S^k = [0,1]$, and state distribution function $F$, which is unknown to the users. 
In the Markovian model, rewards generated by resource $k$ follows a Markovian process with a finite state space $S^k$. The state of resource $k$ evolves according to an irreducible, aperiodic transition probability matrix $P^k$ which is unknown to the users. The stationary distribution of arm $k$ is denoted by $\boldsymbol{\pi}^k = (\pi^k_x)_{x \in S^k}$. 


%

Let $(P^k)'$ denote the {\em adjoint} of $P^k$ on $l_2(\pi)$ where
\begin{align*}
(p^k)'_{xy}=(\pi^k_y p^k_{yx})/\pi^k_x, ~ \forall x,y \in S^k,
\end{align*}
and $\dot{P}^k=(P^k)'P$ denote the {\em multiplicative symmetrization} of $P^k$. Let $\upsilon^k$ be the eigenvalue gap, i.e., $1$ minus the second largest eigenvalue of $\dot{P}^k$, and $\upsilon_{\min} = \min_{k \in {\cal K}} \upsilon^k$. We assume that the $P^k$'s are such that $\dot{P}^k$'s are irreducible. To give a sense of how weak or strong this assumption is, we note that this is a weaker condition than assuming the Markov chains to be reversible. This technical assumption is required in the following large deviation bound that we frequently use in the proofs.

\begin{lemma} \label{intro:lemma:lezaud}
[Theorem 3.3 from \cite{lezaud1998chernoff}] Consider a finite-state, irreducible Markov chain $\left\{X_t\right\}_{t \geq 1}$ with state space $S$, matrix of transition probabilities $P$, an initial distribution $\mathbf{q}$ and stationary distribution $\mathbf{\pi}$. Let $V_{\boldsymbol{q}}=\left\|(\frac{q_x}{\pi_x}, x\in S)\right\|_2$. Let $\dot{P}=P'P$ be the multiplicative symmetrization of $P$ where $P'$ is the adjoint of $P$ on $l_2(\pi)$. Let $\upsilon=1-\lambda_2$, where $\lambda_2$ is the second largest eigenvalue of the matrix $\dot{P}$. $\upsilon$ will be referred to as the eigenvalue gap of $\dot{P}$. Let $f:S\rightarrow \mathbb{R}$ be such that $\sum_{y \in S} \pi_y f(y) =0$, $\left\|f\right\|_{\infty} \leq 1$ and $0<\left\|f\right\|^2_2 \leq 1$. If $\dot{P}$ is irreducible, then for any positive integer $T$ and all $0 < \gamma \leq 1$,
\begin{align*}
P\left(\frac{\sum_{t=1}^T f(X_t)}{T} \geq \gamma \right) \leq V_{\boldsymbol{q}} \exp \left[-\frac{T \gamma^2 \upsilon}{28}\right] ~.
\end{align*}
\end{lemma}

At each time $t$, user $i$ selects a single resource based on the algorithm $\alpha_i$ it uses. Let $\alpha_i(t)$ be the resource selected by user $i$ at time $t$ when it uses algorithm $\alpha_i$. Let $\boldsymbol{\alpha}(t)=\{\alpha_1(t),\alpha_2(t),\ldots,\alpha_M(t)\}$ be the vector of resource selections at time $t$. A resource-usage pair is defined as the tuple $(k,n)$, where $k$ is the index of the resource and $n$ is the number of users using that resource. We consider two different resource allocation problems in the following sections. In the first model, the resource rewards are user-independent. A user $i$ selecting resource $k$ at time $t$ gets reward $R^i_k(t) = r_k(s^k_t,n^k_t)$, where $s^k_t$ is the state of resource $k$ at time $t$, and $n^k_t$ is the number of users using resource $k$ at time $t$. In this model, communication between the users is not possible, but limited feedback about the user activity, i.e., $n^k_t$, is available to each user who selects resource $k$. \revq{For example, when the resources are channels, and users are transmitter-receiver pairs, a user can observe the number of users on the channel it is using by a threshold or feature detector. Besides the limited feedback, each user also knows the total number of users in the system. This can be done by all users initially broadcasting their IDs. This way, if a user knows that the resource rewards are user-independent, it can form estimates of the optimal allocation based on its own observations.} In the second model, resource rewards are user-dependent and communication between the users is possible but incurs some cost. In this model, if user $i$ selects resource $k$ at time $t$, it gets reward $R^i_k(t) = r^i_k(s^k_t,n^k_t)$. \revq{In this case, the user does not need to observe the number of users using the same resource with it, since the joint action profile to be selected can be decided by all users, or it can be decided by some user and announced to other users. The announcement to user $i$ will include $n^k_t$ for the resource it is assigned, and since the users are cooperative, user $i$ is certain that the number of users using the same resource with it remains $n^k_t$ until another joint action profile is decided.}

\revq{In the following sections, we consider distributed learning algorithms to maximize the total reward of all users over any time horizon $T$. The weak regret of a learning algorithm at time $T$ is defined as the difference between the total expected reward of the best static strategy of users up to time $T$, i.e., the strategy in which a users selects the same resource at each time step up to $T$, and the total expected reward of the distributed algorithm up to time $T$. Although, stronger regret measures are available for the Markovian model, where a user may switch resources dynamically in the optimal strategy weak regret is also a commonly used performance measure for this setting \cite{auer2003nonstochastic, tekin2012online, liuliu2011}. A stronger regret measure is considered in \cite{tekin2011adaptive} for a single user bandit problem, but the algorithm to find this solution is computationally inefficient. In this paper we only consider weak-regret, and whenever regret is mentioned it is meant to be weak regret unless otherwise stated.}

\revq{In the following sections, we will show that in order for the system to achieve the optimal order of regret, a user does not need to know the state of the resource it selects. It can learn the optimal strategy by only observing the resource rewards.}

\subsection{User-independent rewards with limited feedback} \label{it_probform_1} 
 
In this model, we assume homogeneous users, i.e., if they happen to select the same resource, they will get the same reward. However, the reward they obtain depends on the interaction among users who select the same resource. Specifically, when user $i$ selects resource $k$ at time $t$, it gets an instantaneous reward $r_k(s^k_t,n^k_t)$. Since the users perceive identical resource qualities, we call this model the {\em symmetric interaction model}. No communication is allowed between the users. However, a user receives feedback about the resource it is using, which is the number of users using that resource. We assume that the rewards are bounded, and without loss of generality 
\begin{align*}
r_k : S^k \times {\cal M} \rightarrow (0,1], \forall k \in {\cal K}. 
\end{align*}
When the reward process of arm $k$ is i.i.d. with distribution $F$, the mean reward of resource-usage pair $(k,n)$ is
\begin{align*}
\mu_{k,n} := \int r_k(s,n) F(ds).
\end{align*}
When the reward process of arm $k$ is Markovian with transition probability matrix $S^k$, the mean reward of resource-usage pair $(k,n)$ is 
\begin{align*}
\mu_{k,n} := \sum_{s \in S^k} \mu^k_s r_k(s,n).
\end{align*}

The set of optimal allocations in terms of number of users using each resource is
\begin{align*}
{\cal B} := \arg\max_{\boldsymbol{n} \in {\cal N}} \sum_{k=1}^K n_k \mu_{k, n_k}, \notag
\end{align*}
where ${\cal N} = \{\boldsymbol{n}=(n_1,n_2,\ldots,n_K): n_k \geq 0, n_1+n_2+\ldots+n_K = M\}$ is the set of possible allocations of users to resources, with $n_k$ being the number of users using resource $k$. Let
\begin{align*}
v(\boldsymbol{n}) := \sum_{k=1}^K n_k \mu_{k, n_k}
\end{align*}
denote the value of allocation $\boldsymbol{n}$. Then, the value of the optimal allocation is
\begin{align*}
v^* := \max_{\boldsymbol{n} \in {\cal N}} v(\boldsymbol{n}). 
\end{align*}
For any allocation $\boldsymbol{n} \in {\cal N}$, the suboptimality gap is defined as
\begin{align*}
\Delta(\boldsymbol{n}) := v^* - \sum_{k=1}^K n_k \mu_{k, n_k}.
\end{align*}
Then the minimum suboptimality gap is 
\begin{align}
\Delta_{\min} := \min_{\boldsymbol{n} \in {\cal N} - {\cal B}} \Delta(\boldsymbol{n}). \label{eqn:minsubopt}
\end{align}
We have the following assumption on the set of optimal allocations:
\begin{assumption} \label{ass:it_probform}
(Uniqueness) There is a unique optimal allocation in terms of the number of users on each channel. The cardinality of ${\cal B}$, i.e., $|{\cal B}|=1$.
\end{assumption}

\rev{Let $\boldsymbol{n}^*$ denote the unique optimal allocation when Assumption \ref{ass:it_probform} holds. This assumption guarantees convergence by random selections over the optimal channels, when each user knows the optimal allocation. Without the uniqueness assumption, even if all users know the set of optimal allocations, convergence by a simple randomizations is not possible. In that case, the users need to bias some of the optimal allocations, in order to converge to one of them. In Section \ref{sec:diss}
, we explain how the uniqueness assumption can be relaxed. } The uniqueness assumption implies the following stability condition.
\begin{lemma}\label{lemma:stability} (Stability)
When Assumption \ref{ass:it_probform} holds, for a set of estimated mean rewards $\hat{\mu}_{k,n_k}$, if $|\hat{\mu}_{k,n_k} - \mu_{k,n_k}| < \Delta_{\min}/2M$, $\forall k \in {\cal K}, n_k \in {\cal M}$ then
\begin{align*}
\arg\max_{\boldsymbol{n} \in {\cal N}} \sum_{k=1}^K n_k \hat{\mu}_{k, n_k} = {\cal B}.
\end{align*}
\end{lemma}
\begin{proof}
Let $\hat{v}(\boldsymbol{n})$ be the estimated value of allocation $\boldsymbol{n}$ computed using the estimated mean rewards $\hat{\mu}_{k,n_k}$. Then, $|\hat{\mu}_{k,n_k} - \mu_{k,n_k}| < \Delta_{\min}/(2M)$, $\forall k \in {\cal K}, n_k \in {\cal M}$ implies that for any $\boldsymbol{n} \in {\cal N}$, we have $| \hat{v}(\boldsymbol{n}) - v(\boldsymbol{n}) | \leq \Delta_{\min}/2$. This implies that 
\begin{align}
v^* - \hat{v}(\boldsymbol{n}^*) < \Delta_{\min}/2, \label{eqnlemmastability1}
\end{align}
and, for any suboptimal allocation $\boldsymbol{n} \in {\cal N}$
\begin{align}
\hat{v}(\boldsymbol{n}) - v(\boldsymbol{n}) < \Delta_{\min}/2. \label{eqnlemmastability2}
\end{align}
By (\ref{eqnlemmastability1}) and (\ref{eqnlemmastability2}), for any suboptimal $\boldsymbol{n}$
\begin{align*}
\hat{v}(\boldsymbol{n}^*) - \hat{v}(\boldsymbol{n}) > \Delta_{\min} - 2 \Delta_{\min}/2 = 0.
\end{align*}
\end{proof}

Stability condition guarantees that when a user estimates the sample mean of the rewards of resource-usage pairs accurately, it can find the optimal allocation. In this section we study the case when a lower bound on $\Delta_{\min}$ is known by the users. This assumption may seem too strong since the users do not know the statistics of the resource rewards. However, if the resource reward represents a discrete quantity such as data rate in bytes or income from resource in dollars then all users will know that $\Delta_{\min} = 1$ byte or dollar. Extension of our results to the case when $\Delta_{\min}$ is not known by the users can be done by increasing the number of samples that are used to form estimates $\hat{\mu}_{k,n}$ over time at a specific rate. In Section \ref{sec:diss}, we investigate this extension in detail.

%
\revq{The regret of algorithm $\boldsymbol{\alpha}$ by time $t$ is given by}
\begin{align}
R^{\boldsymbol{\alpha}}(T) = T v^* - E^{\boldsymbol{\alpha}}\left[\sum_{t=1}^T \sum_{i=1}^M R^i_{\alpha_i(t)}(t) \right]. \label{m:iid:eqn:regret1}
\end{align}

\rev{In order to maximize its reward, a user needs to} compute the optimal allocation based on its estimated mean rewards. This is a combinatorial optimization problem which is NP-hard. We assume that each time a user computes the optimal allocation, a computational cost $C_{cmp}$ is incurred. \rev{For example, this cost can model the time it takes to compute the optimal allocation or the energy consumption of a wireless node associated with the computation. Although in our model we assume that computation is performed at the end of a time slot, we can modify the model such that computation takes finite number of time slots, and $C_{cmp}$ is the total regret due a single computation.}

Sometimes, even when a user learns that some other resource is better than the resource it is currently using, it may be hesitant to switch to the new resource because of the cost incurred by changing the resource. For example, for a radio, switching to another frequency band requires resynchronization of transmitter and receiver. Let $C_{swc}$ be the cost of changing the currently used resource. \rev{We assume that switching is performed at the end of a time slot, but our results still hold if switching requires multiple time slots.}
  
If we add these costs to the regret, it becomes
\begin{align}
R^{\boldsymbol{\alpha}}(T) &= T v^* - E^{\boldsymbol{\alpha}}\left[\sum_{t=1}^T \sum_{i=1}^M R^i_{\alpha_i(t)}(t) \right] +C_{cmp} \sum_{i=1}^M m^i_{cmp} (T) \notag \\ 
&+ C_{swc} \sum_{i=1}^M m^i_{swc} (T) , \label{m:iid:eqn:regret2}
\end{align}
where $m^i_{cmp} (T)$ denotes the number of computations done by user $i$ by time $T$, and $m^i_{swc} (T)$ denotes the number of resource switchings done by user $i$ by time $T$. Then the problem becomes balancing the loss in the performance, the loss due to the NP-hard computation, and the loss due to resource switching.

\revq{Let ${\cal O}^*$ be the set of resources that are used by at least one user in the optimal allocation, and ${\cal O}_i(t)$ be the set of resources that are used by at least one user in the estimated optimal allocation of user $i$.} Let $N^i_{k,n}(t)$ be the number of times user $i$ selected resource $k$ and observed $n$ users on it by time $t$, and $\hat{\mu}^i_{k,n}(t)$ be the sample mean of the rewards collected from resource-usage pair $(k,n)$ by the $t$th play of that pair by user $i$.

\subsection{User-specific rewards with costly communication} \label{it_probform_2} 

In this model, resource rewards are user-specific. Users using the same resource at the same time may receive different rewards. Let $r^i_k(s^k_t,n^k_t)$ be the instantaneous reward user $i$ gets from resource $k$ at time $t$. The expected reward of resource-usage pair $(k,n)$ is given by
\begin{align*}
\mu^i_{k,n} := \int r^i_k(s,n) F(ds),
\end{align*}
for the i.i.d. model, and by
\begin{align*}
\mu_{k,n} := \sum_{s \in S^k} \mu^k_s r_k(s,n),
\end{align*}
for the Markovian model. In this case, the set of optimal allocations is
\begin{align}
{\cal A} := \arg\max_{\boldsymbol{\alpha} \in {\cal K}^M} \sum_{i=1}^M \mu^i_{\alpha_i, n_{\alpha_i}(\boldsymbol{\alpha})}. \label{eqn:it_probform_2_1}
\end{align}
Similar to the definitions in Section \ref{it_probform_1}, $v(\boldsymbol{\alpha})$ is the value of allocation $\boldsymbol{\alpha}$, $v^*$ is the value of an optimal allocation, $\Delta(\boldsymbol{\alpha})$ is the suboptimality gap of allocation $\boldsymbol{\alpha}$ and $\Delta_{\min}$ is the minimum suboptimality gap, i.e.,
\begin{align*}
\Delta_{\min} := v^* - \arg\max_{\boldsymbol{\alpha} \in {\cal K}^M - {\cal A}^*} \Delta(\boldsymbol{\alpha}).
\end{align*}

In order to estimate the optimal allocation, a user must know how resource qualities are perceived by other users. Note that the limited feedback discussed in the previous section is not enough in this case, since it only helps a user to distinguish the rewards of resource-usage pairs of the same resource. Therefore, in this section we assume that communication between the users is possible. \revq{This can either be done by broadcasting every time communication is needed, or broadcasting the next time to communicate on a specific channel initially, and then using time division multiple access on that channel to transmit information about the resource estimates, and next time step to communicate.} Every time a user communicates with other users, it incurs cost $C_{com}$. Considering the computation cost $C_{cmp}$ of computing (\ref{eqn:it_probform_2_1}), and the switching cost $C_{swc}$ the regret at time $T$ is 
\begin{align}
R^{\boldsymbol{\alpha}}(T) &:= T w^* - E^{\boldsymbol{\alpha}}\left[\sum_{t=1}^T \sum_{i=1}^MR^i_{\alpha_i(t)}(t) \right] +C_{cmp} \sum_{i=1}^M m^i_{cmp} (T) \notag \\
&+ C_{swc} \sum_{i=1}^M m^i_{swc} (T) + C_{com} \sum_{i=1}^M m^i_{com} (T)  , \label{eqn:it_probform_2_2}
\end{align}
where $m^i_{com} (T)$ is the number of times user $i$ communicated with other users by time $T$.

The following stability condition, which is the analogue of Lemma \ref{lemma:stability}, states that when the estimated rewards of resource-usage pairs are close to their true rewards, the set of estimated optimal allocations is equal to the set of optimal allocations.

\begin{lemma}\label{lemma:it_probform1} (Stability)
For a set of estimated mean rewards $\hat{\mu}^i_{k,n}$, if $|\hat{\mu}^i_{k,n} - \mu^i_{k,n}| < \Delta_{\min}/(2M)$, $\forall i \in {\cal M}, k \in {\cal K}, n \in {\cal M}$ then,
\begin{align*}
\arg\max_{\boldsymbol{\alpha} \in {\cal K}^M} \sum_{i=1}^M \hat{\mu}^i_{\alpha_i, n_{\alpha_i}(\boldsymbol{\alpha})} = {\cal A}.
\end{align*}
\end{lemma}
\begin{proof}
The proof is similar to the proof of Lemma \ref{lemma:stability}. Note that the value of each allocation is sum of $M$ resource-usage pairs.   
\end{proof}

In the following sections we will present an algorithm that achieves logarithmic regret. However, this algorithm requires that a lower bound on $\Delta_{\min}$ is known by the users. However, generally such a parameter is unknown to the users since the users learn the mean resource rewards over time. In Section \ref{sec:diss}, we propose a method to solve this problem, which gives a near-logarithmic regret result. 
\section{A distributed synchronized algorithm for user-independent rewards} \label{sec:fb_alg}

In this section we propose a distributed algorithm for the users which achieves logarithmic regret with respect to the optimal allocation without communication by using the partial feedback described in the previous section. We will show that this algorithm achieves logarithmic regret both in the i.i.d. and the Markovian  resource models. Our algorithm is called {\em Distributed Learning with Ordered Exploration} (DLOE), whose pseudocode is given in Figure \ref{m:iid2:fig:RCA3}. The DLOE algorithm uses the idea of {deterministic sequencing of exploration and exploitation with initial synchronization}, first proposed in \cite{liuliu2011}, to achieve logarithmic regret for the \revq{Markovian rewards}.  A key difference between the problem studied here and that in \cite{liuliu2011} is that the latter assumes that the optimal allocation of users to resources is orthogonal, i.e., there can be at most one user using a resource in the optimal allocation.  For this reason the \revq{technical development} in \cite{liuliu2011} are not applicable to the general symmetric interaction model introduced in this paper.

\subsection{Definition of DLOE}

DLOE operates in blocks. Each block is either an exploration block or an exploitation block. The length of an exploration (exploitation) block geometrically increases in the number of exploration (exploitation) blocks. The parameters that determine the length of the blocks is the same for all users. User $i$ has a fixed exploration sequence ${\cal N}_i$, which is a sequence of resources to select, of length $N'$. \rev{In its $l$th exploration block user $i$ selects resources in ${\cal N}_i$ in a sequential manner by selecting a resource for $c^{l-1}$ time slots before proceeding to the next resource in the sequence. The $z$th resource in sequence ${\cal N}_i$ is denoted by ${\cal N}_i(z)$.} The set of sequences ${\cal N}_1, \ldots, {\cal N}_M$ is created in a way that all users will observe all resource-usage pairs at least once, and that the least observed resource-usage pair for each user is observed only once in a single (parallel) run of the set of sequences by the users. \revq{For example when $M=2$, $K=2$, exploration sequences ${\cal N}_1 = \{1, 1, 2, 2\}$ and ${\cal N}_2 = \{1,2,1,2\}$ are sufficient for each user to sample all resource-usage pairs by once.} \revq{Note that it is always possible to find such a set of sequences. With $M$ users and $K$ resources, there are $K^M$ possible assignments of users to resources. Index each of the $K^M$ possible assignments as $\{ \boldsymbol{\alpha}(1) ,\boldsymbol{\alpha}(2),\ldots, \boldsymbol{\alpha}(K^M)\}$. Then, using the set of sequences ${\cal N}_i = \{ \alpha_i(1), \alpha_i(2), \ldots, \alpha_i(K^M) \}$ for $i \in {\cal M}$, all users sample all resource-usage pairs by at least once.}


\revq{The sequence ${\cal N}_i$ is assumed known to user $i$ before the resource selection process starts. Let $l^i_O(t)$ be the number of completed exploration blocks and $l^i_I(t)$ be the number of completed exploitation blocks of user $i$ by time $t$, respectively. For usesr $i$, the length of the $l$th exploration block is $N' c^{l-1}$ and the length of the $l$th exploitation block is $ab^{l-1}$, where $a,b,c$ are positive integers greater than $1$.

At the beginning of each block, user $i$ computes $N^i_O(t) := \sum_{l=1}^{l^i_O(t)} c^{l-1}$. If $N^i_O(t) \geq L \log t$, user $i$ starts an exploitation block at time $t$. Otherwise, it starts an exploration block. Here $L$ is the exploration constant which controls the number of explorations. Clearly, the number of exploration steps up to $t$ is non-decreasing in $L$. Since the estimates of mean rewards of resource-usage pairs are formed by sample mean estimates of observations during exploration steps, by increasing the value of $L$, a user can control the probability of deviation of estimated mean rewards from the true mean rewards. Intuitively, $L$ should be chosen according to $\Delta_{\min}$, since the accuracy of estimated value of an allocation depends on the accuracy of estimated mean rewards.}

Because of the deterministic nature of the blocks and the property of the sequences ${\cal N}_1, \ldots, {\cal N}_M$ discussed above, if at time $t$ a user starts a new exploration (exploitation) block, then all users start a new exploration (exploitation) block. \revq{Therefore $l^i_O(t) = l^j_O(t)$, $N^i_O(t) = N^j_O(t)$, $l^i_I(t) = l^j_I(t)$, for all $i,j \in {\cal M}$. Since these quantities are equal for all users we drop the superscripts and denote them by $l_O(t), N_O(t), l_I(t)$.} Let $t_l$ be the time at the beginning of the $l$th exploitation block. At time $t_l$, $l=1,2,\ldots$, user $i$ computes an estimated optimal allocation $\boldsymbol{\hat{n}}^i(l) = \{\hat{n}^i_1(l), \ldots, \hat{n}^i_K(l)\}$ based on the sample mean estimates of the resource-usage pairs given by
\begin{align*}
\boldsymbol{\hat{n}}^i(l) = \arg\max_{\boldsymbol{n} \in {\cal N}} \sum_{k=1}^K n_k \hat{\mu}_{k, n_k}(N^i_{k, n_k}(t_l)),
\end{align*}
and chooses a resource from the set ${\cal O}_i(l)$ which is the set of resources selected by at least one user in $\boldsymbol{\hat{n}}^i(l)$. During the $l$th exploitation block, in order to settle to an optimal configuration, if the number of users on the resource $\alpha_i(t)$ user $i$ selects is greater than $\hat{n}^i_{\alpha_i(t)}(l)$, then in the next time step the user $i$ randomly chooses a resource within ${\cal O}_i(l)$. \revq{The probability that user $i$ chooses resource $k \in {\cal O}_i(l)$ is $\hat{n}^i_k/ |{\cal O}_i(l)|$. Therefore, it is more likely for a user to select a resource which it believes a large number of users use in the optimal allocation than a resource which it believes is used by smaller number of users.}
\revq{This kind of randomization guarantees that the users settle to the optimal allocation in finite expected time, given that 
the estimated the optimal allocation at the beginning of the exploitation block is equal to the optimal allocation.}

\begin{figure}[ht]
\fbox {
\begin{minipage}{0.95\columnwidth}
{Distributed Learning with Ordered Exploration (DLOE) for user $i$}
\begin{algorithmic}[1]
\STATE{Initialize: $t=1$, $l_O = 0$, $l_I = 0$, $\eta=1$, $X_O=0$, $F=2$, $z=1$, $len=2$, $\hat{\mu}^i_{k,n}=0$, $N^i_{k,n}=0$, $\forall k \in {\cal K}, n \in \{1,2,\ldots,M\}$ $a,b,c \in \{2,3,\ldots\}$.}
\WHILE{$t \geq 1$}
\IF{$F=1$ //Exploitation block}
\IF{$\hat{n}^i_{\alpha_i(t-1)}(t-1) > \hat{n}^i_{\alpha_i(t-1)}$}
\STATE{Pick $\alpha_i(t)$ randomly from ${\cal O}_i$ with $P(\alpha_i(t) = k) \hat{n}^i_{k}/|{\cal O}_i|$.}
\ELSE
\STATE{$\alpha_i(t) = \alpha_i(t-1)$}
\ENDIF
\IF{$\eta = len$}
\STATE{$F=0$}
\ENDIF
\ELSIF{$F=2$ //Exploration block}
\STATE{$\alpha_i(t) = {\cal N}_i'(z)$}
\IF{$\eta = len$}
\STATE{$\eta=0$, $++z$}
\ENDIF
\IF{$z=|{\cal N}_i| + 1$}
\STATE{$X_O = X_O + len$}
\STATE{$F=0$}
\ENDIF
\ELSE
\STATE{//IF $F=0$}
\IF{$X_O \geq L \log t$}
\STATE{//Start an exploitation epoch}
\STATE{$F=1$, $++l_I$, $\eta=0$, $len = a \times b^{l_I - 1}$}
\STATE{//Compute the estimated optimal allocation}
\STATE{$\boldsymbol{\hat{n}}^i = \arg\max_{\boldsymbol{n} \in {\cal N}} \sum_{k=1}^K \hat{\mu}_{k, n_k} I(n_k \neq 0)$}
\STATE{Set ${\cal O}_i$ to be the set of resource in $\boldsymbol{\tilde{n}}^i$ with at least one user.}
\ELSIF{$X_O < L \log t$}
\STATE{//Start an exploration epoch}
\STATE{$F=2$, $++l_O$, $\eta=0$, $len = c^{l_O - 1}$, $z=1$}
\ENDIF
\ENDIF
\STATE{Let $l_i(t)$ be the number of users on resource $\alpha_i(t)$}
\STATE{$++N^i_{\alpha_i(t),l_i(t)}$}
\STATE{\begin{align*}
\hat{\mu}^i_{\alpha_i(t),l_i(t)} &= \frac{(N^i_{\alpha_i(t),l_i(t)}-1)\hat{\mu}^i_{\alpha_i(t),l_i(t)}+ r^i_{\alpha_i(t),l_i(t)}(t)}{N^i_{\alpha_i(t),l_i(t)}}
\end{align*}}
\STATE{$++\eta$, $++t$}
\ENDWHILE
\end{algorithmic}
\end{minipage}
}
\caption{pseudocode of DLOE} \label{m:iid2:fig:RCA3}
\end{figure}

\subsection{Analysis of the regret of DLOE}

Next, we analyze the regret of DLOE. There are three factors contributing to the regret in both i.i.d. and the Markovian models. The first one is the regret due to exploration blocks, the second one is the regret due to incorrect computation of the optimal allocation by a user, and the third one is the regret due to the randomization before settling to the optimal allocation given that all users computed the optimal allocation correctly.
\revq{In addition to those terms contributing to the regret, another contribution to the regret in the Markovian model comes from the transient effect of a resource-usage pair not being in its stationary distribution when chosen by a user. This arises because we compare the performance of DLOE in which users dynamically change their selections with the mean reward of the optimal static allocation.}

In the following lemmas, we will bound the parts of the regret that is common to both i.i.d. and the Markovian models. Bounds on the parts of the regret, which depends on the resource model, are given in the next two subsections.

If a user starts exploration block at time $t$, then 
\begin{align}
\sum_{l=1}^{l_O(t)} c^{l-1} < L \log t &\Rightarrow \frac{c^{l_O(t)} - 1}{c-1} < L \log t \notag  \\
&\Rightarrow c^{l_O(t)} < (c-1) L \log t + 1  \notag \\
&\Rightarrow l_O(t) < \log_c ((c-1) L \log t) + 1. \label{m:iid2:eqn:expbound1}
\end{align}
Let $T_O(t)$ be the time spent in exploration blocks by time $t$. By (\ref{m:iid2:eqn:expbound1}) we have
\begin{align}
T_O(t) &\leq \sum_{l=1}^{l_O(t)} N' c^{l-1} < N' \frac{c^{l_O(t)-1}}{c-1}
\leq \frac{N' (c-1) L \log t}{c-1} = N' L \log t. \label{m:iid2:eqn:expbound2}
\end{align}

\begin{lemma} \label{m:iid2:lemma:regret_explore}
For any $t>0$ in an exploitation block, regret due to explorations by time $t$ is at most
\begin{align*}
M N' L \log t.
\end{align*}
\end{lemma}
\begin{proof}
\revq{Due to the bounded rewards in $(0,1]$, an upper bound to the worst case is when each user loses a reward of $1$ due to suboptimal decisions at each step in an exploration block.} Result follows the bound (\ref{m:iid2:eqn:expbound2}) for $T_O(t)$.
\end{proof}
By time $t$ at most $t-N'$ slots have been spent in exploitation blocks (because of initial exploration first $N'$ slots are always in an exploration block). Therefore
\begin{align}
\sum_{l=1}^{l_I(t)} ab^{l-1} &= a \frac{b^{l_I(t)-1}}{b-1} \leq t - N' \notag \\
& \Rightarrow b^{l_I(t)} \leq \frac{b-1}{a} (t - N') \notag \\
& \Rightarrow l_I(t) \leq \log_b \left(\frac{b-1}{a} (t-N')\right). \label{m:iid2:eqn:expbound3}
\end{align}

Next lemma bounds the computational cost of solving the NP-hard optimization problem of finding the estimated optimal allocation.
\begin{lemma} \label{m:iid2:lemma:computation}
When users use DLOE, the regret due to the computations by time $t$ is upper bounded by
\begin{align*}
C_{cmp} M \log_b \left(\frac{b-1}{a} (t-N')\right),
\end{align*}
\end{lemma}
\begin{proof}
Optimal allocation is computed at the beginning of each exploitation block. Number of exploitation blocks is bounded by (\ref{m:iid2:eqn:expbound3}), hence the result follows.
\end{proof}

\subsection{Analysis of regret for the i.i.d. problem} \label{tekin:sec:iid2-sync}

In this subsection we analyze the regret of DLOE in the i.i.d. model. We note that in the i.i.d. model the reward of each resource-usage pair is generated by an i.i.d. process with support in $(0,1]$. Apart from the bounds on the parts of the regret given in Section \ref{sec:fb_alg}, our next step is to bound the regret caused by incorrect calculation of the optimal allocation by some user by using a Chernoff-Heoffding bound. Let $\epsilon := \Delta_{\min}/(2M)$ denote the maximum distance between the estimated resource-usage reward and the true resource-usage reward such that Lemma \ref{lemma:stability} holds. 
\begin{lemma} \label{m:iid2:lemma:regret_exploit}
Under the i.i.d. model, when each user uses DLOE with constant $L \geq 1/\epsilon^2$, regret due to incorrect calculations of the optimal allocation at the beginning of the $l$th exploitation block is at most
\begin{align*}
M^3 K (log(t_l) + 1).
\end{align*}
\end{lemma}
\begin{proof}
Let $H(t_l)$ be the event that at the beginning of the $l$th exploitation block, there exists at least one user who computed the optimal allocation incorrectly. Let $\omega$ be a sample path of the stochastic process generated by the learning algorithm and the stochastic resource rewards. The event that user $i$ computes the optimal allocation incorrectly is a subset of the event
\begin{align*}
\{|\hat{\mu}^i_{k,n}(N^i_{k,n}(t_l)) - \mu^i_{k,n}| \geq \epsilon \textrm{ for some } k \in {\cal K}, n \in {\cal M}  \}.
\end{align*}
Therefore $H(t_l)$ is a subset of the event
\begin{align*}
\{|\hat{\mu}^i_{k,n}(N^i_{k,n}(t_l)) - \mu^i_{k,n}| \geq \epsilon \textrm{ for some } i \in {\cal M}, k \in {\cal K}, n \in {\cal M}  \}.
\end{align*}
Using a union bound, we have
\begin{align*}
I(\omega \in H(t_l)) \leq \sum_{i=1}^M \sum_{k=1}^K \sum_{n=1}^M I (|\hat{\mu}^i_{k,n}(N^i_{k,n}(t_l)) - \mu^i_{k,n}| \geq \epsilon).
\end{align*}
Then taking its expected value, we get
\begin{align}
P(\omega \in H(t_l))
&\leq \sum_{i=1}^M \sum_{k=1}^K \sum_{n=1}^M P(|\hat{\mu}^i_{k,n}(N^i_{k,n}(t_l)) - \mu^i_{k,n}| \geq \epsilon). \label{m:iid2:eqn:inequality1}
\end{align}
Since 
\begin{align*}
P(|a-b| \geq \epsilon) = 2 P (a-b \geq \epsilon),
\end{align*}
for $a,b > 0$, we have
\begin{align}
P(|\hat{\mu}^i_{k,n}(N^i_{k,n}(t_l)) - \mu^i_{k,n}| \geq \epsilon)
&= 2 P(\hat{\mu}^i_{k,n}(N^i_{k,n}(t_l)) - \mu^i_{k,n} \geq \epsilon). \label{m:iid2:eqn:equality1}
\end{align}
Since $l$ is an exploitation block we have $N^i_{k,n}(t_l) > L \log t_l$. \revq{Let $r^i_{k,n}(t) := r_k(s^k_t, n^k_t)$ and let $\tilde{t}^i_{k,n}(l)$ denote the time when user $i$ chooses resource $k$ and observes $n$ users on it for the $l$th time.} We have
\begin{align*}
\hat{\mu}^i_{k,n}(N^i_{k,n}(t_l)) = \frac{\sum_{z=1}^{N^i_{k,n}(t_l)} r^i_{k,n}(\tilde{t}^i_{k,n}(z))}{N^i_{k,n}(t_l)}
\end{align*}
Using a Chernoff-Hoeffding bound
\begin{align}
P(\hat{\mu}^i_{k,n}(N^i_{k,n}(t_l)) - \mu^i_{k,n} \geq \epsilon)
&= P\left(\sum_{z=1}^{N^i_{k,n}(t_l)} r^i_{k,n}(\tilde{t}^i_{k,n}(z)) \geq N^i_{k,n}(t_l) \mu^i_{k,n} + N^i_{k,n}(t_l) \epsilon \right) \notag \\
&\leq e^{-2 N^i_{k,n}(t_l) \epsilon^2 } \leq e^{-2 L \log t_l \epsilon^2} = \frac{1}{t^2_l}, \label{m:iid2:eqn:timebound}
\end{align}
where the last equality follows from the fact that $L \geq 1/\epsilon^2$. Substituting (\ref{m:iid2:eqn:equality1}) and (\ref{m:iid2:eqn:timebound}) into (\ref{m:iid2:eqn:inequality1}), we have
\begin{align}
P(\omega \in H(t_l)) \leq M^2 K \frac{1}{t^2_l}. \label{m:iid2:eqn:hbound}
\end{align}
The regret in the $l$th exploitation block caused by incorrect calculation of the optimal allocation by at least one user is upper bounded by
\begin{align}
M ab^{l-1} P(\omega \in H(t_l)), \notag
\end{align}
since there are $M$ users and the resource rewards are in $(0,1]$.
From (\ref{m:iid2:eqn:expbound3}) we have $ab^{l-1} \leq t_l$. Therefore the regret caused by incorrect calculation of the optimal allocation by at least one user by time $t_l$ is
\begin{align}
\sum_{z=1}^{l} M ab^{l-1} P(\omega \in H(t_l)) \leq M^3 K \sum_{z=1}^{l} ab^{l-1} \frac{1}{t^2_l}
\leq M^3 K \sum_{z=1}^l \frac{1}{t_z} &\leq \notag \\
 M^3 K \sum_{t=1}^{t_l} \frac{1}{t} \leq M^3 K (\log(t_l) + 1). \notag
\end{align}
\end{proof}

The following lemma bounds the expected number of exploitation blocks where some user computed the optimal allocation incorrectly.
\begin{lemma} \label{m:iid2:lemma:numberbound1}
When users use DLOE with $L \geq 1/\epsilon^2$, the expected number of exploitation blocks up to any $t$ in which there exists at least one user who computed the optimal allocation wrong is bounded by
\begin{align*}
E\left[\sum_{l=1}^\infty I(\omega \in H(t_l))\right] \leq \sum_{l=1}^\infty \frac{M^2K}{t^2_l} \leq M^2 K \beta,
\end{align*}
where $\beta = \sum_{t=1}^\infty 1/t^2$.
\end{lemma}
\begin{proof}
Proof is similar to the proof of Lemma \ref{m:iid2:lemma:regret_exploit}, using the bound (\ref{m:iid2:eqn:hbound}) for $P(\omega \in H(t_l))$.
\end{proof}

Finally, we bound the regret due to the randomization before settling to the optimal allocation in exploitation slots in which all users have computed the optimal allocation correctly.

\begin{lemma} \label{lemma:fb_perf1_1}
Denote the number of resources which are selected by at least one user in the optimal allocation by $z^*$. Reindex the resources in ${\cal O}^*$ by $\{1,2,\ldots, z^*\}$. Let ${\cal M} = \{ \boldsymbol{m} :  m_1 + m_2 + \ldots + m_{z^*} \leq M, m_i \geq 0, \forall i \in \{1,2,\ldots, z^*\} \} $. 
In an exploitation block where each user computed the optimal allocation correctly, the regret in that block \revq{due to randomizations before settling to the optimal allocation in this block} is upper bounded by
\begin{align*}
& O_B := \frac{1}{\min_{ \boldsymbol{m} \in {\cal M}} P_{DLOE}(\boldsymbol{m})},
\end{align*}
where
\begin{align*}
P_{DLOE}(\boldsymbol{m}) = \frac{(M-m)!}{(n_1 - m_1)! \ldots (n_{z^*} - m_{z^*})!} \left( \frac{n_1}{M}  \right)^{n_1-m_1} \ldots \left( \frac{n_{z^*}}{M}  \right)^{n_{z^*} - m_{z^*}}.
\end{align*}
\end{lemma}
\begin{proof}
Consider user $i$ an an exploitation block $l$ in which it knows the optimal allocation, i.e., $\hat{\boldsymbol{n}}^i(l) = \boldsymbol{n}^*$. Since user $i$ does not know the selections of other users, knowing the optimal allocation is not enough for users to jointly select the optimal allocation. At time $t$ in exploitation block $l$, user $i$ selects a resource $k \in {\cal O}_i(t)$, then observes $n^k_t$. If $n^k_t \leq n^*_k$ it selects the same channel at $t+1$. Otherwise, it selects a channel randomly from ${\cal O}^*$ at $t+1$. The probability that user $i$ selects channel $k$ is $n^*_k/|{\cal O}^*|$.

Note that this randomization is done at every slot by at least one user until the joint allocation of user's corresponds to the optimal allocation or until the exploitation block ends. Next, we will show that when the users follow DLOE, at every slot, the probability that the joint allocation settles to the optimal allocation is grater than $p_W$ which is the probability of settling to the optimal allocation in the worst configuration of users on channels. 
Note that $n_1 + n_2 + \ldots + n_{z^*} = M$. Consider the case where $m$ of the users do not randomize while the others randomize. Let $\boldsymbol{m} = (m_1, m_2, \ldots, m_{z^*})$ be the number of users on each resource in ${\cal O}^*$ who does not randomize.  We have $m = m_1 + m_2 + \ldots + m_{z^*}$. The probability of settling to the optimal allocation in a single round of randomizations is
\begin{align*}
P_{DLOE}(\boldsymbol{m}) = \frac{(M-m)!}{(n_1 - m_1)! \ldots (n_{z^*} - m_{z^*})!} \left( \frac{n_1}{M}  \right)^{n_1-m_1} \ldots \left( \frac{n_{z^*}}{M}  \right)^{n_{z^*} - m_{z^*}}.
\end{align*}
Where $M!/(n^*_1! n^*_2! \ldots n^*_{z^*}!)$ is the number allocations $\boldsymbol{\alpha}$ which results in the unique optimal allocation in terms of number of users using resources, and 
\begin{align*}
\left( \frac{n_1}{M}  \right)^{n_1 -m_1} \ldots \left( \frac{n_{z^*}}{M}  \right)^{n_{z^*} - m_{z^*}},
\end{align*}
is the probability that such an allocation happens. Then, 
\begin{align*}
p_W = \min_{ \boldsymbol{m} \in {\cal M}} P_{DLOE}(\boldsymbol{m}).
\end{align*}

Let $\boldsymbol{n}_I$ be the allocation at the beginning of an exploitation block in which all users computed the optimal allocation correctly. Let $p_{\boldsymbol{n}_I}(t)$ be the probability of settling to the optimal allocation in $t$th round of randomizations in this exploitation block. Then the expected number of steps before settling to the optimal allocation is
\begin{align*}
E[t_{op}] = \sum_{t=1}^\infty t p_{\boldsymbol{n}_I}(t) \prod_{i=1}^{t-1} (1 - p_{\boldsymbol{n}_I}(i)).
\end{align*}
Since $p_W \leq p_{\boldsymbol{n}_I}(t)$ for all $\boldsymbol{n}_I$ and $t$, we have
\begin{align*}
E[t_{op}] \leq \sum_{t=1}^\infty  t p_w (1-p_w)^{t-1} = 1/p_w.
\end{align*}
\end{proof}

\begin{lemma} \label{m:iid2:lemma:regret3}
The regret due to randomization before settling to the optimal allocation is bounded by
\begin{align*}
O_B M^3 K \beta,
\end{align*}
where
\begin{align*}
\beta = \sum_{t=1}^\infty \frac{1}{t^2},
\end{align*}
\end{lemma}
\begin{proof}
A {\em good} exploitation block is an exploitation block in which all the users computed the optimal allocation correctly. A {\em bad} exploitation block is a block in which there exists at least one user who computed the optimal allocation incorrectly. The worst case is when each bad block is followed by a good block. The number of bad blocks is bounded by Lemma \ref{m:iid2:lemma:numberbound1}. After each such transition from a bad block to a good block, \revq{the} expected loss is at most $O_B$, which is given in Lemma \ref{lemma:fb_perf1_1}.
\end{proof}

\begin{lemma} \label{lemma:fb_perf1_2}
When users use DLOE, for any $t>0$ which is the beginning of an exploitation block, the regret due to switchings by time $t$ is upper bounded by
\begin{align*}
C_{swc} M  \left( N' L \log t  + O_B \log_b \left(\frac{b-1}{a} (t-N')\right) + M^2 K (\log t + 1) \right).
\end{align*}
\end{lemma}
\begin{proof}
By Lemma \ref{m:iid2:lemma:regret_explore}, \revq{the} time spent in exploration blocks by $t$ is bounded by $N' L \log t$. Since resource rewards is always in $(0,1]$ at most $N' L \log t$ expected regret per user can result \revq{from} explorations. Number of exploitation blocks is bounded by $\log_b \left(\frac{b-1}{a} (t-N')\right) $. If an exploitation block is a good block as defined in Lemma \ref{m:iid2:lemma:regret3}, users will settle to the optimal allocation in $O_B$ expected \revq{time slots}. By time $t$, there are at most $\log_b \left(\frac{b-1}{a} (t-N')\right)$ exploitation blocks. Hence, the regret in good exploitation blocks cannot be larger than $C_{swc} O_B \log_b \left(\frac{b-1}{a} (t-N')\right)$. If an exploitation block is a bad block, by Lemma \ref{m:iid2:lemma:regret_exploit}, \revq{the} expected number of slots spent in such exploitation blocks is upper bounded by $M^2 K (log(t_l) + 1)$. Assuming in the worst case all users switch at every slot in a bad block, regret due to switching in bad exploitation blocks is upper bounded by $C_{swc} M^3 (log(t_l) + 1)$.
\end{proof}

Combining all the results above we have the following theorem.
\begin{theorem} \label{m:iid2:maintheorem}
If all users use DLOE with $L \geq 1/\epsilon^2$, at the beginning of $l$th exploitation block, the regret defined in (\ref{m:iid:eqn:regret1}) is upper bounded by,
\begin{align*}
(M N'L + M^3 K) \log(t_l) + M^3 K (\beta O_B + 1),
\end{align*}
and the regret defined in (\ref{m:iid:eqn:regret2}) is upper bounded by
\begin{align*}
&M^3 K (\log t_l + 1) (1 + C_{swc}) + M N' L \log t_l (1 + C_{swc}) \\
&+ M \log_b \left(\frac{b-1}{a} (t-N')\right) (C_{swc} O_B + C_{cmp}) + O_B M^3 K \beta,
\end{align*}
where $O_B$ given in Lemma \ref{lemma:fb_perf1_1}, \revq{is the worst case expected hitting time of the optimal allocation given all users know the optimal allocation, and $\beta = \sum_{t=1}^\infty 1/t^2$.}
\end{theorem}
\begin{proof}
The result follows from summing the regret terms from Lemmas \ref{m:iid2:lemma:regret_explore}, \ref{m:iid2:lemma:regret_exploit}, \ref{m:iid2:lemma:regret3}, \ref{m:iid2:lemma:computation} and Lemma \ref{lemma:fb_perf1_2}.
\end{proof} 
\subsection{Analysis of regret for the Markovian problem} \label{tekin:sec:restless-sync}

In this subsection we analyze the regret of DLOE \revq{in case of Markovian rewards}. \revq{The analysis in this section is quite different from the analysis in Section \ref{tekin:sec:iid2-sync} due to the Markovian rewards.}
Apart from the bounds on the parts of the regret given in Section \ref{it_fb_alg}, our next step is to bound the regret caused by incorrect calculation of the optimal allocation by some user. Although the proof of following the lemma is very similar to the proof of Lemma \ref{m:iid2:lemma:regret_exploit}, due to the Markovian nature of the rewards, we need to bound a large deviation probability from the sample mean of each resource-usage pair for multiple contiguous segments of observations from that resource-usage pair. For the simplicity of analysis, we assume that DLOE is run with parameters $a=2$, $b=4$, $c=4$, for the Markovian rewards. Similar analysis can be done for other, arbitrary parameter values. Let $\epsilon := \Delta_{\min}/(2M)$. 

\begin{lemma} \label{m:restless:lemma:regret_exploit}
Under the Markovian model, when each user uses DLOE with constant
\begin{align*}
L \geq \max \{1/\epsilon^2, 50 S^2_{\max} r_{\Sigma, \max}^2/ ((3-2\sqrt{2})\upsilon_{\min})\},
\end{align*}
the regret due to incorrect calculations of the optimal allocation at the beginning of the $l$th exploitation block is at most
\begin{align*}
M^3 K \left( \frac{1}{\log 2} + \frac{\sqrt{2L}} {10 r_{\Sigma, \min}} \right) \frac{S_{\max}}{\pi_{\min}} (log(t_l) + 1).
\end{align*}
\end{lemma}
\begin{proof}
\revq{Similar to the analysis for the i.i.d. rewards,} let $H(t_l)$ be the event that at the beginning of the $l$th exploitation block, there exists at least one user who computed the optimal allocation incorrectly, and let $\omega$ be a sample path of the stochastic process generated by the learning algorithm and the stochastic rewards. Proceeding the same way as in the proof of Lemma \ref{m:iid2:lemma:regret_exploit} by (\ref{m:iid2:eqn:inequality1}) and (\ref{m:iid2:eqn:equality1}) we have,
\begin{align}
P(\omega \in H(t_l))
&\leq \sum_{i=1}^M \sum_{k=1}^K \sum_{n=1}^M 2 P(\hat{\mu}^i_{k,n}(N^i_{k,n}(t_l)) - \mu^i_{k,n} \geq \epsilon). \label{m:restless:eqn:inequality1}
\end{align}
Since $t_l$ is the beginning of an exploitation block we have $N^i_{k,n}(t_l) \geq L \log t_l$, $\forall i \in {\cal M}, k \in {\cal K}, n \in {\cal M}$. This implies that $N^i_{k,n}(t_l) \geq \sqrt{N^i_{k,n}(t_l) L \log t_l}$. Hence
\begin{align}
&P(\hat{\mu}^i_{k,n}(N^i_{k,n}(t_l)) - \mu^i_{k,n} \geq \epsilon) \notag \\
&= P(N^i_{k,n}(t_l) \hat{\mu}^i_{k,n}(N^i_{k,n}(t_l)) - N^i_{k,n}(t_l) \mu^i_{k,n} \geq \epsilon N^i_{k,n}(t_l) ) \notag \\
&\leq P \left( N^i_{k,n}(t_l) \hat{\mu}^i_{k,n}(N^i_{k,n}(t_l)) - N^i_{k,n}(t_l) \mu^i_{k,n} \geq \epsilon \sqrt{N^i_{k,n}(t_l) L \log t_l}  \right) \label{m:restless:eqn:qinq}
\end{align}
To bound (\ref{m:restless:eqn:qinq}), we proceed the same way as in the proof of Theorem 1 in \cite{liuliu2010}. The idea is to separate the total number of observations of the resource-usage pair $(k,n)$ by user $i$ into multiple contiguous segments. Then, using a union bound, (\ref{m:restless:eqn:qinq}) is upper bounded by the sum of the deviation probabilities for each segment. By a suitable choice of the exploration constant $L$, the deviation probability in each segment is bounded by a negative power of $t_n$. Combining this with the fact that the number of such segments is logarithmic in time (due to the geometrically increasing block lengths), for block length parameters $a=2$, $b=4$, $c=4$ in DLOE, and for
\begin{align}
L \geq \max \{1/\epsilon^2, 50 S^2_{\max} r_{\Sigma, \max}^2/ ((3-2\sqrt{2})\upsilon_{\min})\}, \notag
\end{align}
we have,
\begin{align}
&P \left(N^i_{k,n}(t_l) \hat{\mu}^i_{k,n}(N^i_{k,n}(t_l)) - N^i_{k,n}(t_l) \mu^i_{k,n} \geq \epsilon \sqrt{N^i_{k,n}(t_l) L \log t_l}  \right) \notag \\
&\leq \left( \frac{1}{\log 2} + \frac{\sqrt{2L}} {10 r_{\Sigma, \min}} \right) \frac{S_{\max}}{\pi_{\min}} t_l^{-2}. \notag
\end{align}
Continuing from (\ref{m:restless:eqn:inequality1}), we get
\begin{align}
P(\omega \in H(t_l)) \leq M^2 K \left( \frac{1}{\log 2} + \frac{\sqrt{2L}} {10 r_{\Sigma, \min}} \right) \frac{S_{\max}}{\pi_{\min}} t_l^{-2}. \label{m:restless:eqn:hbound}
\end{align}
The result is obtained by continuing the same way as in the proof of Lemma \ref{m:iid2:lemma:regret_exploit}.
\end{proof}

The following lemma bounds the expected number of exploitation blocks where some user computed the optimal allocation incorrectly.
\begin{lemma} \label{m:restless:lemma:numberbound1}
Under the Markovian model, when each user uses DLOE with constant
\begin{align*}
L \geq \max \{1/\epsilon^2, 50 S^2_{\max} r_{\Sigma, \max}^2/ ((3-2\sqrt{2})\upsilon_{\min})\},
\end{align*}
the expected number exploitation blocks up to any $t$ in which there exists at least one user who computed the optimal allocation wrong is bounded by
\begin{align*}
E\left[\sum_{l=1}^\infty I(\omega \in H(t_l))\right]  \leq M^2 K \left( \frac{1}{\log 2} + \frac{\sqrt{2L}} {10 r_{\Sigma, \min}} \right) \frac{S_{\max}}{\pi_{\min}} \beta,
\end{align*}
where $\beta = \sum_{t=1}^\infty 1/t^2$.
\end{lemma}
\begin{proof}
Proof is similar to the proof of Lemma \ref{m:restless:lemma:regret_exploit}, using the bound (\ref{m:restless:eqn:hbound}) for $P(\omega \in H(t_l))$.
\end{proof}

Next, we bound the regret due to the randomization before settling to the optimal allocation in exploitation slots in which all users have computed the optimal allocation correctly.
\begin{lemma} \label{m:restless:lemma:regret3}
The regret due to randomization before settling to the optimal allocation is bounded by
\begin{align*}
(O_B + C_{\boldsymbol{P}})  M^3 K \left( \frac{1}{\log 2} + \frac{\sqrt{2L}} {10 r_{\Sigma, \min}} \right) \frac{S_{\max}}{\pi_{\min}} \beta,
\end{align*}
where $O_B$ given in Lemma \ref{lemma:fb_perf1_1}, \revq{is the worst case expected hitting time of the optimal allocation given all users know the optimal allocation, $\beta = \sum_{t=1}^\infty 1/t^2$, and $C_{\boldsymbol{P}} = \max_{k \in {\cal K}} C_{P^k}$ where $C_P$ is a constant that depends on the transition probability matrix $P$.}
\end{lemma}
\begin{proof}
A {\em good} exploitation block is an exploitation block in which all the users computed the optimal allocation correctly. A {\em bad} exploitation block is a block in which there exists at least one user who computed the optimal allocation incorrectly. By converting the problem into a simple balls in bins problem where the balls are users and the bins are resources, the expected number of time slots spent before settling to the optimal allocation in a good exploitation block is bounded above by $O_B$. The worst case is when each bad block is followed by a good block, and the number of bad blocks is bounded by Lemma \ref{m:restless:lemma:numberbound1}. Moreover, due to the transient effect that a resource may not be at its stationary distribution when it is selected, even after settling to the optimal allocation in an exploitation block the regret of at most $C_{\boldsymbol{P}}$ can be accrued by a user. \revq{This is because the difference between the $t$-horizon expected reward of an irreducible, aperiodic Markov chain with an arbitrary initial distribution and $t$ times the expected reward at the stationary distribution is bounded by $C_{\boldsymbol{P}}$ independent of $t$.} Since there are $M$ users and resource rewards are in $[0,1]$, the result follows.
\end{proof}

Similar to the i.i.d. case, the next lemma bounds the regret due to switchings in the Markovian case. 

\begin{lemma} \label{lemma:fb_perf2_2}
When users use DLOE, for any $t>0$ which is the beginning of an exploitaiton block, the regret due to the switchings by time $t$ is upper bounded by
\begin{align*}
C_{swc} M  \left( N' L \log t  + O_B \log_b \left(\frac{b-1}{a} (t-N')\right) + M^2 K \left( \frac{1}{\log 2} + \frac{\sqrt{2L}} {10 r_{\Sigma, \min}} \right) \frac{S_{\max}}{\pi_{\min}} (\log t + 1) \right).
\end{align*}
\end{lemma}
\begin{proof}
This proof is similar to the proof of Lemma \ref{lemma:fb_perf1_2} for the i.i.d. model. By Lemma \ref{m:iid2:lemma:regret_explore}, the time spent in exploration blocks by time $t$ is bounded by $N' L \log t$. Since rewards are always in $[0,1]$, at most $N' L \log t$ expected regret per user can result due to explorations. The number of exploitation blocks is bounded by $\log_b \left(\frac{b-1}{a} (t-N')\right) $. If an exploitation block is a good block as defined in Lemma \ref{m:restless:lemma:regret3}, users will settle to the optimal allocation in $O_B$ expected steps. By time $t$, there are at most $\log_b \left(\frac{b-1}{a} (t-N')\right)$ exploitation blocks. Hence, the regret due to switchings in good exploitation blocks cannot be larger than $C_{swc} O_B \log_b \left(\frac{b-1}{a} (t-N')\right)$. If an exploitation block is a bad block, by Lemma \ref{m:restless:lemma:regret_exploit}, the expected number of slots spent in such exploitation blocks is upper bounded by
\begin{align*}
M^2 K \left( \frac{1}{\log 2} + \frac{\sqrt{2L}} {10 r_{\Sigma, \min}} \right) \frac{S_{\max}}{\pi_{\min}} (log(t_l) + 1).
\end{align*}
Assuming in the worst case all users switch at every slot in a bad block, the regret due to switchings in bad exploitation blocks is upper bounded by
\begin{align*}
C_{swc} M^3 K \left( \frac{1}{\log 2} + \frac{\sqrt{2L}} {10 r_{\Sigma, \min}} \right) \frac{S_{\max}}{\pi_{\min}} (log(t_l) + 1).
\end{align*}
\end{proof}

Combining all the results above we have the following theorem.
\begin{theorem} \label{m:restless:maintheorem1}
Under the Markovian model, when each user uses DLOE with constant
\begin{align*}
L \geq \max \{1/\epsilon^2, 50 S^2_{\max} r_{\Sigma, \max}^2/ ((3-2\sqrt{2})\upsilon_{\min})\},
\end{align*}
then at the beginning of $l$th exploitation block, the regret defined in (\ref{m:iid:eqn:regret1}) is upper bounded by
\begin{align*}
&\left(M N' L + M^3 K \left( \frac{1}{\log 2} + \frac{\sqrt{2L}} {10 r_{\Sigma, \min}} \right) \frac{S_{\max}}{\pi_{\min}} \right) \log(t_l) \\
&+ M^3 K \left( \frac{1}{\log 2} + \frac{\sqrt{2L}} {10 r_{\Sigma, \min}} \right) \frac{S_{\max}}{\pi_{\min}}  ( \beta(O_B + C_{\boldsymbol{P}}) + 1),
\end{align*}
and the regret defined in (\ref{m:iid:eqn:regret2}) is upper bounded by
\begin{align*}
& M N' L \log t_l (1+ C_{swc}) + M^3 K \left( \frac{1}{\log 2} + \frac{\sqrt{2L}} {10 r_{\Sigma, \min}} \right) \frac{S_{\max}}{\pi_{\min}} (1 + C_{swc}) (\log t_l +1 ) \\
&+ M (C_{cmp} + C_{swc} O_B) \log_4 \left(\frac{3}{2} (t_l-N')\right)  \\
&+ (O_B + C_{\boldsymbol{P}})  M^3 K \left( \frac{1}{\log 2} + \frac{\sqrt{2L}} {10 r_{\Sigma, \min}} \right) \frac{S_{\max}}{\pi_{\min}} \beta, 
\end{align*}
where $O_B$ given in Lemma \ref{lemma:fb_perf1_1}, \revq{is the worst case expected hitting time of the optimal allocation given all users know the optimal allocation, $\beta = \sum_{t=1}^\infty 1/t^2$, and $C_{\boldsymbol{P}} = \max_{k \in {\cal K}} C_{P^k}$ where $C_P$ is a constant that depends on the transition probability matrix $P$.}
\end{theorem}
\begin{proof}
The result follows from summing the regret terms from Lemmas \ref{m:iid2:lemma:regret_explore}, \ref{m:restless:lemma:regret_exploit}, \ref{m:restless:lemma:regret3}, \ref{m:iid2:lemma:computation} and \ref{lemma:fb_perf2_2}, and the fact that $a=2$, $b=4$.
\end{proof}

Our results show that when initial synchronization between users is possible, logarithmic regret, which is the optimal order of regret even in the centralized case can be achieved. \revq{Moreover, the proposed algorithm does not need to know whether the rewards are i.i.d. or Markovian. It achieves logarithmic regret in both cases.}

\section{A distributed synchronized algorithm for user-specific rewards} \label{sec:com_alg}

In this section we consider the model where the resource rewards are user-dependent. Different from the previous section, where a user can compute the socially optimal allocation based only on its own estimates, with user-dependent rewards each user needs to know the estimated rewards of other users in order to compute the optimal allocation. We assume that users can communicate with each other, but this communication incurs a cost $C_{com}$. For example, in an OSA model, users are transmitter-receiver pairs that can communicate with each other on one of the available channels, even when no common control channel exists. In order for communication to take place, each user can broadcast a request for communication over all available channels. For instance, if users are using an algorithm based on deterministic sequencing of exploration and exploitation, then at the beginning, a user can announce the parameters that are used to determine the block lengths. \revq{This way, the users can decide on which exploration and exploitation sequences to use, so that all of them can start an exploration block or exploitation block at the same time.}
After this initial communication, before exploitation block, users share their perceived channel qualities with each other, and one of the users, which can be chosen in a round robin fashion, computes the optimal allocation and announces to each user the resource it should select in the optimal allocation.
\revq{Next, we propose the algorithm distributed learning with communication (DLC) for this model.}

\subsection{Definition of DLC}

Similar to DLOE, DLC (see Figure \ref{fig:it_alg_com}) consists of geometrically increasing exploration and exploitation blocks. The predetermined exploration order allows each user to observe the reward from each resource-usage pair, and update the sample mean rewards of each resource-activity pair. \revq{Note that in this case, since feedback about $n^k_t$ is not needed, each user should be given the number of users using the same resource with it for each resource in the predetermined exploration order.} Similar to the previous section, this predetermined exploration order can be seen as an input to the algorithm from the algorithm designer. On the other hand, since communication between the users is possible, the predetermined exploration order can be determined by a user and then communicated to the other users, or users may collectively reach to an agreement over a predetermined exploration order by initial communication. In both cases, the initial communication will incur a constant cost $C_i$.

Let ${\cal N}_i$ be the exploration sequence of user $i$, which is defined the same way as in Section \ref{sec:fb_alg}, and ${\cal L}_i(z)$ be the number of users using the same resource with user $i$ in the $z$th slot of an exploration block.
Based on the initialization methods discussed above both ${\cal N}_i$ and ${\cal L}_i$ are known by the user $i$ at the beginning. Note that different from DLOE, DLC only uses the resource-usage pair reward estimates from the exploration blocks. 
However, if the user who computed the optimal allocation announces the resources assigned to other users as well, then rewards from the exploitation blocks can also be used to update the resource-usage pair reward estimates. 

Our regret analysis in the following sections shows that estimates from the exploration blocks alone are sufficient to achieve logarithmic regret.

\begin{figure}[ht]
\fbox {
\begin{minipage}{0.95\columnwidth}
{Distributed Learning with Communication (DLC) for user $i$}
\begin{algorithmic}[1]
\STATE{Initialize: $t=1$, $l_O = 0$, $l_I = 0$, $\eta=1$, $X_O=0$, $F=2$, $z=1$, $len=2$, $\hat{\mu}^i_{k,n}=0$, $N^i_{k,n}=0$, $\forall k \in {\cal K}, n \in \{1,2,\ldots,M\}$ $a,b,c \in \{2,3,\ldots\}$.}
\WHILE{$t \geq 1$}
\IF{$F=1$ //Exploitation block}
\STATE{Select resource $\alpha_i(t) = \alpha^*_i$.}
\IF{$\eta = len$}
\STATE{$F=0$}
\ENDIF
\ELSIF{$F=2$ //Exploration block}
\STATE{$\alpha_i(t) = {\cal N}_i'(z)$}
\STATE{$++N^i_{\alpha_i(t),{\cal L}_i(z)}$}
\STATE{\begin{align*}
\hat{\mu}^i_{\alpha_i(t),l_i(t)} &= \frac{(N^i_{\alpha_i(t),{\cal L}_i(z)}-1)\hat{\mu}^i_{\alpha_i(t),l_i(t)}+ r^i_{\alpha_i(t),l_i(t)}(t)}{N^i_{\alpha_i(t),{\cal L}_i(z)}}
\end{align*}}
\IF{$\eta = len$}
\STATE{$\eta=0$, $++z$}
\ENDIF
\IF{$z=|{\cal N}_i| + 1$}
\STATE{$X_O = X_O + len$}
\STATE{$F=0$}
\ENDIF
\ELSE
\STATE{//IF $F=0$}
\IF{$X_O \geq L \log t$}
\STATE{//Start an exploitation epoch}
\STATE{$F=1$, $++l_I$, $\eta=0$, $len = a \times b^{l_I - 1}$}
\STATE{//Communicate estimated channel qualities $\hat{\mu}^i_{k,n}$, $k \in {\cal K}$, $n \in {\cal M}$ with other users.}
\IF{ $(l_I \mod M)+1 = i$}
\STATE{//Compute the estimated optimal allocation, and sent each other user the resource it should use in the exploitation block.}
\STATE{$\boldsymbol{\alpha}^* = \arg\max_{\boldsymbol{\alpha} \in {\cal K}^M} \sum_{i=1}^M \hat{\mu}^i_{\alpha_i, n_{\alpha_i}(\boldsymbol{\alpha})}$}
\ELSE
\STATE{//Receive the resource that will be user $\alpha^*_i$ in the exploitation block from user $(l_I \mod M)+1$.}
\ENDIF
\ELSIF{$X_O < L \log t$}
\STATE{//Start an exploration epoch}
\STATE{$F=2$, $++l_O$, $\eta=0$, $len = c^{l_O - 1}$, $z=1$}
\ENDIF
\ENDIF
%
\STATE{$++\eta$, $++t$}
\ENDWHILE
\end{algorithmic}
\end{minipage}
}
\caption{pseudocode of DLC} \label{fig:it_alg_com}
\end{figure}

\subsection{Analysis of the regret of DLC}

In this section we bound the regret terms which are same for both i.i.d. and Markovian resource rewards.

\begin{lemma} \label{it_com_alg:lemma1}
For any $t>0$ which is in an exploitation block, regret of DLC due to explorations by time $t$ is at most
\begin{align*}
M N' L \log t.
\end{align*}
\end{lemma}
\begin{proof}
Since DLC uses deterministic sequencing of exploration and exploitation the same way as DLOE, the proof is same as the proof of Lemma \ref{m:iid2:lemma:regret_explore} by using the bound (\ref{m:iid2:eqn:expbound2}) for $T_O(t)$.
\end{proof}

Since communication takes place at the beginning of each exploitation block, it can be computed the same way as computation cost is computed for DLOE. Moreover, since resource switching is only done during exploration blocks or at the beginning of a new exploitation block, switching costs can also be computed the same way. The following lemma bounds the communication, computation and switching cost of DLC.

\begin{lemma} \label{it_com_alg:lemma2}
When users use DLC, at the beginning of the $l$th exploitation block, the regret terms due to communication, computation and switching are upper bounded by
\begin{align*}
C_{com} M \log_b \left(\frac{b-1}{a} (t_l-N')\right)+ C_i, \\
C_{cmp} M \log_b \left(\frac{b-1}{a} (t_l-N')\right), \\
C_{swc} M \left( \log_b \left(\frac{b-1}{a} (t_l-N')\right) +  N' L \log t_l \right),
\end{align*}
respectively, where $C_i$ is the cost of initial communication.
\end{lemma}
\begin{proof}
Communication is done initially and at the beginning of exploitation blocks. Computation is only performed at the beginning of exploitation blocks. Switching is only done at exploration blocks or at the beginning of exploitation blocks. Number of exploitation blocks is bounded by (\ref{m:iid2:eqn:expbound3}), and time slots in exploration blocks is bounded by (\ref{m:iid2:eqn:expbound2}).
\end{proof} 

In the next subsections we analyze the parts of regret that are different for i.i.d. and Markovian rewards.
\subsection{Analysis of regret for the i.i.d. problem} \label{it_com_perf1}

In this subsection we analyze the regret of DLC in the i.i.d. model. The analysis is similar with the user-independent reward model given in Section \ref{sec:fb_alg}.
\begin{lemma} \label{lemma:it_com_perf1_1}
Under the i.i.d. model, when each agent uses DLC with constant $L \geq 1/\epsilon^2$, regret due to incorrect calculations of the optimal allocation at the beginning of the $l$th exploitation block is at most
\begin{align*}
M^3 K (log(t_l) + 1).
\end{align*}
\end{lemma}
\begin{proof}
Let $H(t_l)$ be the event that at the beginning of the $l$th exploitation block, the estimated optimal allocation calculated by user $(l \mod M)+1$ is different from the true optimal allocation. Let $\omega$ be a sample path of the stochastic process generated by the learning algorithm and the stochastic arm rewards. The event that user $(l \mod M)+1$ computes the optimal allocation incorrectly is a subset of the event
\begin{align*}
\{|\hat{\mu}^i_{k,n}(N^i_{k,n}(t_l)) - \mu^i_{k,n}| \geq \epsilon \textrm{ for some } i \in {\cal M}, k \in {\cal K}, n \in {\cal M}  \}.
\end{align*}
Analysis follows from using a union bound, taking the expectation, and then using a Chernoff-Hoeffding bound. Basically, it follows from (\ref{m:iid2:eqn:inequality1}) in Lemma \ref{m:iid2:lemma:regret_exploit}.
\end{proof}

The following lemma bounds the expected number of exploitation blocks in which the optimal allocation is calculated incorrectly.
\begin{lemma} \label{lemma:it_com_perf1_2}
When agents use DLC with $L \geq 1/\epsilon^2$, the expected number exploitation blocks up to any $t$ in which the optimal allocation is calculated incorrectly is bounded by
\begin{align*}
E\left[\sum_{l=1}^\infty I(\omega \in H(t_l))\right] \leq \sum_{l=1}^\infty \frac{M^2K}{t^2_l} \leq M^2 K \beta,
\end{align*}
where $\beta = \sum_{t=1}^\infty 1/t^2$.
\end{lemma}
\begin{proof}
Please see the proof of Lemma \ref{m:iid2:lemma:numberbound1}.
\end{proof}

Combining all the results above we have the following theorem.
\begin{theorem} \label{m:iid2:maintheorem}
If all agents use DLC with $L \geq 1/\epsilon^2$, at the beginning of $l$th exploitation block, the regret is upper bounded by,
\begin{align*}
(M N'L (1 + C_{swc}) + M^3 K) \log(t_l) + (C_{com} + C_{cmp} + C_{swc}) M \log_b \left(\frac{b-1}{a} (t_l-N')\right) + M^3 K.
\end{align*}
\end{theorem}
\begin{proof}
The result follows from combining the results of Lemmas \ref{it_com_alg:lemma1}, \ref{it_com_alg:lemma2}, \ref{lemma:it_com_perf1_1} and \ref{lemma:it_com_perf1_2}.
\end{proof} 

\subsection{Analysis of regret for the Markovian problem} \label{tekin:sec:restless-sync}

We next analyze the regret of DLC in the Markovian model. The analysis in this section is similar to the ones in Section \ref{tekin:sec:restless-sync}. We assume that DLC is run with parameters $a=2$, $b=4$, $c=4$.

\begin{lemma} \label{lemma:it_com_perf2_1}
Under the Markovian model, when each user uses DLC with constant
\begin{align*}
L \geq \max \{1/\epsilon^2, 50 S^2_{\max} r_{\Sigma, \max}^2/ ((3-2\sqrt{2})\upsilon_{\min})\},
\end{align*}
the regret due to incorrect calculations of the optimal allocation at the beginning of the $l$th exploitation block is at most
\begin{align*}
M^3 K \left( \frac{1}{\log 2} + \frac{\sqrt{2L}} {10 r_{\Sigma, \min}} \right) \frac{S_{\max}}{\pi_{\min}} (log(t_l) + 1).
\end{align*}
\end{lemma}
\begin{proof}
The proof follows the proof of Lemma \ref{m:restless:lemma:regret_exploit}. 
\end{proof}

The following lemma bounds the expected number of exploitation blocks where some user computed the optimal allocation incorrectly.
\begin{lemma} \label{lemma:it_com_perf2_2}
Under the Markovian model, when each user uses DLC with constant
\begin{align*}
L \geq \max \{1/\epsilon^2, 50 S^2_{\max} r_{\Sigma, \max}^2/ ((3-2\sqrt{2})\upsilon_{\min})\},
\end{align*}
the expected number of exploitation blocks up to any $t$ in which there exists at least one user who computed the optimal allocation wrong is bounded by
\begin{align*}
E\left[\sum_{l=1}^\infty I(\omega \in H(t_l))\right]  \leq M^2 K \left( \frac{1}{\log 2} + \frac{\sqrt{2L}} {10 r_{\Sigma, \min}} \right) \frac{S_{\max}}{\pi_{\min}} \beta,
\end{align*}
where $\beta = \sum_{t=1}^\infty 1/t^2$.
\end{lemma}
\begin{proof}
The proof follows the proof of Lemma \ref{m:restless:lemma:numberbound1}.
\end{proof}

Combining all the results above we have the following theorem.
\begin{theorem} \label{m:restless:maintheorem1}
Under the Markovian model, when each user uses DLC with constant
\begin{align*}
L \geq \max \{1/\epsilon^2, 50 S^2_{\max} r_{\Sigma, \max}^2/ ((3-2\sqrt{2})\upsilon_{\min})\},
\end{align*}
then at the beginning of the $l$th exploitation block, the regret is upper bounded by
\begin{align*}
&\left( M N' L (1+C_{swc}) + M^3 K \left( \frac{1}{\log 2} + \frac{\sqrt{2L}} {10 r_{\Sigma, \min}} \right) \frac{S_{\max}}{\pi_{\min}} \right) \log(t_l) \\
&+ (C_{com} + C_{cmp} + C_{swt}) M \log_4 \left(\frac{3}{2} (t_l-N')\right) + M^3 K \left( \frac{1}{\log 2} + \frac{\sqrt{2L}} {10 r_{\Sigma, \min}} \right) \frac{S_{\max}}{\pi_{\min}}  ( \beta(C_{\boldsymbol{P}}) + 1),
\end{align*}
where $C_{\boldsymbol{P}} = \max_{k \in {\cal K}} C_{P^k}$ where $C_P$ is the constant that depends on the transition probability matrix $P$.
\end{theorem}
\begin{proof}
The result follows from combining results of Lemmas \ref{it_com_alg:lemma1}, \ref{it_com_alg:lemma2}, \ref{lemma:it_com_perf2_1} and \ref{lemma:it_com_perf2_2}, and the fact that $a=2$, $b=4$. Note that due to the transient effect that a resource may not be at its stationary distribution when it is selected, even when all users select resources according to the optimal allocation, a deviation of at most $C_{\boldsymbol{P}}$ from the expected total reward of the optimal allocation is possible. Therefore, at most $C_{\boldsymbol{P}}$ regret results from the transient effects in exploitation blocks where the optimal allocation is calculated correctly. The last term in the regret is a result of this.
\end{proof}

\section{Numerical Results} \label{sec:num}

In this section we consider an opportunistic spectrum access problem in a cognitive radio network consisting of $K=3$ channels and $M=3$ users. We model the primary user activity on each channel as an i.i.d. Bernoulli process with $\theta_k$ being the probability that there is no primary user on channel $k$. At each time step a secondary user senses a channel and transmits with code division multiple access (CDMA) scheme if there is no primary user on that channel. Therefore, when there is no primary user present, the problem reduces to a multi-channel CDMA wireless power control problem. If channel $k$ is not occupied by a primary user, the rate secondary user $i$ gets can be modeled by (see, e.g. \cite{tekin2012-2}),

\begin{eqnarray*}
\log \left(1 + \gamma \frac{h^k_{ii} P^k_i}{N_o + \sum_{j \neq i} h^k_{ji} P^k_j}\right),
\end{eqnarray*}
where $h^k_{ji}$ is the channel gain between transmitter of user $j$ and receiver of user $i$, $P^k_j$ is the transmit power of user $j$ on channel $k$, $N_o$ is the noise power, and $\gamma > 0$ is the spreading gain.

If a primary user is present on channel $k$, in order not to cause interference, secondary users should not transmit on that channel hence they get zero reward from that channel. We assume that the rate function is user-independent, i.e., $h^k_{ii} = \hat{h}^k, \forall i \in {\cal M}$, $h^k_{ji} = \tilde{h}^k, \forall i \neq j \in {\cal M}$, $P^k_i = P^k$, $\forall i \in {\cal M}$. Values of the parameters of the rate functions and primary user activity are given in Table \ref{table:newce}. We assume that $N_o=1$ and $\gamma=1$. In the optimal allocation under these values, the number of users on channels 1,2,3 is 0,2,1 respectively which is not an orthogonal allocation.

Each user applies the DLOE algorithm in a decentralized way. In Figure \ref{fig:figure1}, we plot regret/$\log t$ of DLOE under both definitions of the regret given in (\ref{m:iid:eqn:regret1}) and (\ref{m:iid:eqn:regret2}), respectively. Our results are averaged over $10$ runs of the algorithm. We took $\epsilon$ to be the half of the difference between value of the optimal allocation and the second best allocation which is $0.0811$. We simulate for two values of the exploration constant $L=1/\epsilon^2 = 152$ and $L=4/\epsilon^2 = 608$.  The computational cost of calculating the estimated optimal allocation is $C_{cmp}=100$. We observe that there is a constant difference between the two plots due to the logarithmic number of computations of the optimal allocation. In both plots, we observe an initial linear increase in regret. This is due to the fact that users start exploiting only after they have sufficiently many explorations, i.e., $X_O(t) > L \log t$. It can be deduced that if this difference decreases, a smaller $\epsilon$ should be chosen, which will result in a larger $L$ hence increase in the number of explorations. This will cause the initial linear region to expand, incresing the regret.

In Table \ref{table:newce2}, we give the percentage of times the optimal allocation is played. We observe that the users settle to the channels which are used in the optimal allocation, by estimating the optimal allocation correctly and by randomizing if there are more users on the channel than there should be in the estimated optimal allocation. The largest contribution to the regret comes from the initial exploration blocks. We see that after the initial exploration blocks the percentage of times steps in which the optimal allocation is played increases significantly, up to $90\%$ for $L=152$ and up to $60\%$ for $L=608$. Although the estimates of the mean rewards of channel-activity pairs is more accurate with a larger $L$, more time is spent in exploration, thus the average number of plays of the optimal allocation is smaller. It is clear that under both cases as time goes to infinity, the percentage of time slots in which optimal allocation is played approaches $100\%$. In Table \ref{table:newce3}, we observe the percentage of time steps that a channel is selected by user 1. We see that after the initial explorations, user 1 chooses most of the time the channels that are used by at least one user in the optimal allocation. We see that the percentage of times channel 1 is selected by user 1 falls below $4\%$, for $L=152$ and $15\%$ for $L=608$ at time $5 \times 10^5$.

\begin{table}
\begin{center}
\begin{tabular}{|r|c|c|c|}
\hline
channel & $1$ & $2$ & $3$\\
\hline
$\theta_k$ & $1/8$ & $1/3$ & $1/5$\\
\hline
$\hat{h}^k$ & $5$ & $10$ & $15$\\
\hline
$\tilde{h}^k$ & $1$ & $1.2$ & $3$\\
\hline
$P^k$ & $1$ & $1$ & $1$\\
\hline
\end{tabular}\medskip
\caption{Simulation parameters}
\label{table:newce}
\end{center}
\end{table}

\begin{figure}
\includegraphics[width=\columnwidth]{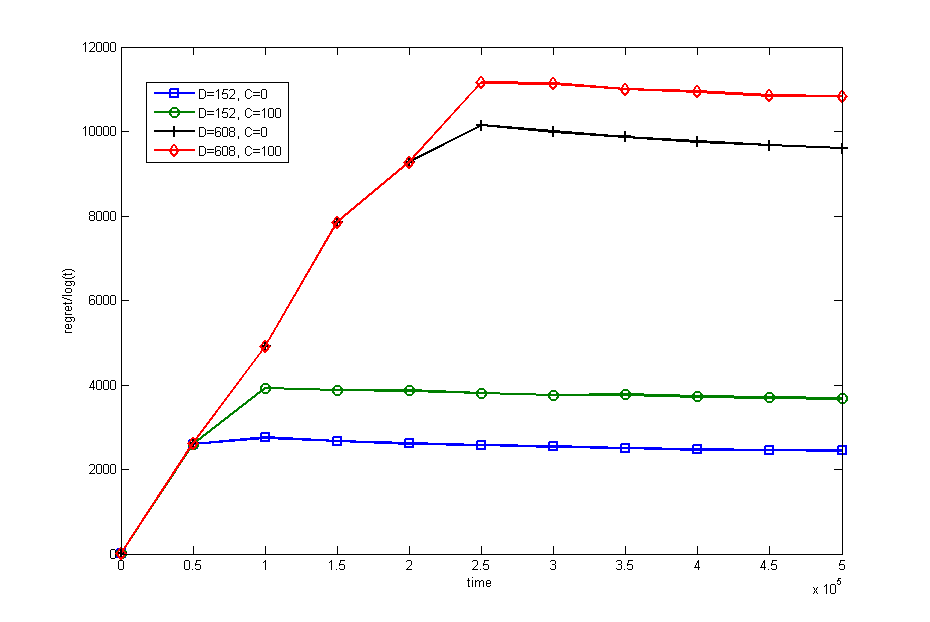}
\caption{Regret/$\log t$ for DLOE with and without computational cost.}
\label{fig:figure1}
\end{figure}

\begin{table}
\begin{center}
\begin{tabular}{|r|c|c|c|c|c|}
\hline
$t$ & $10^2$ & $10^3$ & $10^4$ & $10^5$ & $5 \times 10^5$\\
\hline
$L=152$ & $9$ & $9$ & $8$ & $50$ & $90$\\
\hline
$L=608$ & $9$ & $9$ & $7$ & $10$ & $60$\\
\hline
\end{tabular}\medskip
\caption{Percentage of time the optimal allocation is selected up to $t$.}
\label{table:newce2}
\end{center}
\end{table}

\begin{table}
\begin{center}
\begin{tabular}{|r|c|c|c|c|c|}
\hline

$t (D=152)$ & $10^2$ & $10^3$ & $10^4$ & $10^5$ & $5 \times 10^5$\\
\hline
Channel 1 & $46$ & $44$ & $46$ & $18$ & $4$\\
\hline
Channel 2 & $27$ & $28$ & $31$ & $41$ & $48$\\
\hline
Channel 3 & $27$ & $28$ & $23$ & $41$ & $48$\\
\hline
$t (D=608)$ & $10^2$ & $10^3$ & $10^4$ & $10^5$ & $5 \times 10^5$\\
\hline
Channel 1 & $46$ & $44$ & $46$ & $37$ & $15$\\
\hline
Channel 2 & $27$ & $28$ & $31$ & $37$ & $65$\\
\hline
Channel 3 & $27$ & $28$ & $23$ & $26$ & $20$\\
\hline
\end{tabular}\medskip
\caption{Percentage of channels selected by user 1 up to $t$.}
\label{table:newce3}
\end{center}
\end{table} 
\section{Discussion} \label{sec:diss}

In this section we comment on extensions of our algorithms to more general settings and relaxation of some assumptions we introduced in the previous sections. 

\subsection{Unknown sub-optimality gap}

Both algorithms DLOE and DLC requires that users know a lower bound $\epsilon$ on the difference between the estimated and true mean resource rewards for which the estimated and true optimal allocations coincide. Knowing this lower bound, DLOE and DLC chooses an {\em exploration constant} $L \geq 1/\epsilon^2$ so that after $N' L \log t$ time steps spent in exploration is enough to have reward estimates that are withing $\epsilon$ of the true rewards with a very high probability.

However, $\epsilon$ depends on the suboptimality gap $\delta$ which is a function of the true mean resource rewards which is unknown to the users at the beginning. This problem can be solved in the following way. Instead of using a constant {\em exploration constant} $L$, DLOE and DLC uses an increasing exploration constant $L(t)$ such that $L(1)=1$ and $L(t) \rightarrow \infty$ as $t \rightarrow \infty$. By this way the requirement that $L(t) \geq 1/\epsilon^2$ is satisfied after some finite number of time steps which we denote by $T_0$. In the worst case, $M T_0$ regret will come from these time steps where $L(t) < 1/\epsilon^2$. After $T_0$, only a finite (time-independent) regret will result from incorrect calculations of the optimal allocation due to the inaccuracy in estimates. Since both DLOE and DLC explores only if the least explored resource-congestion pair is explored less than $L(t) \log t$ times, regret due to explorations will be bounded by $M N' L(t) \log t$. Since the order of explorations with $L(t)$ is greater than with constant $L$, the order of exploitations is less than the case with constant $L$. Therefore, the order of regret due to incorrect calculations of the optimal allocation, switchings in exploitation blocks, computation and communication at the beginning of exploitation blocks after $T_0$ is less than the corresponding regret terms when $L$ is constant. Only the regret due to switchings in exploration blocks increases to $C_{swc} M N' L(t) \log t$. Therefore, instead of having $O(\log t)$ regret, without a lower bound on $\epsilon$, the proposed modification achieves $O(L(t) \log t)$ regret.

\subsection{Multiple optimal allocations}

For user-specific rewards with costly communication, the user who computed the estimated optimal allocation announces to other users which resources they should select. Since the users are cooperative and follow the rules of the algorithm, even if there are multiple optimal allocations, they will all use the allocation communicated in this way. The problem arises when there are multiple optimal allocations in the user-independent rewards with limited feedback case. According to DLOE, if there are multiple optimal allocations, even though if all users correctly find an optimal allocation, they may not pick the same optimal allocation since they cannot communicate with each other. To avoid this problem, we proposed Assumption \ref{ass:it_probform}, which guarantees the uniqueness of the optimal allocation. We now desribe a modification on DLOE so this assumption is no longer required.

\revq{Let $\epsilon := \Delta_{\min}/(2M)$, where $\Delta_{\min}$ is the minimum suboptimality gap given in (\ref{eqn:minsubopt}).} Consider user $i$, and an exploitation block $l$. For a resource which is in the set of estimated optimal resources ${\cal O}_i(l)$, if user $i$ observes at the last time slot of that exploitation block, number of users less than or equal to the number of users in the estimated optimal allocation on that resource, at the begining of the next exploitation block, it will increase the total estimated reward of that allocation by $\epsilon/4$. If this is not true, then it will decrease the total estimated reward of that allocation by $\epsilon/4$. Note that when the estimated rewards are accurate enough, i.e., within $\epsilon/2$ of the true resource rewards with a very high probability, even with an $\epsilon/4$ subsidy, a suboptimal allocation will not be chosen. Moreover by this way, among the optimal allocations, users will settle to the resources for which the number of users using it is at most the number of users using it in the optimal allocation. By this modification, after the mean reward estimates are accurate enough, the users will settle to one of the optimal allocations in finite expected number of exploitation blocks.

\section{Conclusion} \label{sec:conc}
In this paper, we proposed distributed online learning algorithms for decentralized multi-user resource sharing problems. We analyzed the performance of our algorithms, and proved that they achieve logarithmic regret under both i.i.d. and Markovian resource reward models when communication, computation and switching costs are present. We presented numerical analysis of dynamic spectrum access application  which is a resource sharing problem.
\bibliographystyle{IEEEtran}
\bibliography{cemreferences}

\end{document}